\def\year{2020}\relax
\documentclass[letterpaper]{article} 
\usepackage{aaai20}  
\usepackage{times}  
\usepackage{helvet} 
\usepackage{courier}  
\usepackage[hyphens]{url}  
\usepackage{graphicx} 
\usepackage{graphicx}  
\usepackage[caption=false,font=normalsize,labelfont=sf,textfont=sf]{subfig}
\usepackage{amsmath}
\usepackage{booktabs}
\usepackage{algorithm}
\usepackage{algpseudocode}
\usepackage{multirow}
\newlength{\without}

\usepackage{amsthm, amssymb}
\usepackage[thinc]{esdiff}
\theoremstyle{definition}
\newtheorem{definition}{Definition}
\theoremstyle{plain}
\newtheorem{theorem}{Theorem}
\newtheorem{proposition}{Proposition}
\newtheorem{lemma}{Lemma}
\newtheorem{assumption}{Assumption}
\usepackage{mathtools}

\title{Trading Convergence Rate with Computational Budget in High Dimensional Bayesian Optimization}
\author{
\\ \Large \textbf{Hung Tran-The, Sunil Gupta, Santu Rana, Svetha Venkatesh}\\ 
Applied Artificial Intelligence Institute,Deakin University, Geelong, Australia\\ 
\{hung.tranthe, sunil.gupta, santu.rana, svetha.venkatesh\}@deakin.edu.au 
}
 
\begin{document}

\maketitle
\begin{abstract}
Scaling Bayesian optimisation (BO) to high-dimensional search spaces is a active and open research problems particularly when no assumptions are made on function structure. The main reason is that at each iteration, BO requires to find global maximisation of acquisition function, which itself is a non-convex optimization problem in the original search space. With growing dimensions, the computational budget for this maximisation gets increasingly short leading to inaccurate solution of the maximisation. This inaccuracy adversely affects both the convergence and the efficiency of BO.  We propose a novel approach where the acquisition function only requires maximisation on a discrete set of low dimensional subspaces embedded in the original high-dimensional search space. Our method is free of any low dimensional structure assumption on the function unlike many recent high-dimensional BO methods.  Optimising acquisition function in low dimensional subspaces allows our method to obtain accurate solutions within limited computational budget. We show that in spite of this convenience, our algorithm remains convergent. In particular, cumulative regret of our algorithm only grows sub-linearly with the number of iterations.  More importantly, as evident from our regret bounds, our algorithm provides a way to trade the convergence rate with the number of subspaces used in the optimisation. Finally, when the number of subspaces is "sufficiently large", our algorithm's cumulative regret is at most  $\mathcal{O}^{*}(\sqrt{T\gamma_T})$ as opposed to $\mathcal{O}^{*}(\sqrt{DT\gamma_T})$ for the GP-UCB of Srinivas et al. (2012), reducing a crucial factor $\sqrt{D}$ where $D$ being the dimensional number of input space. We perform empirical experiments to evaluate our method extensively, showing that its sample efficiency is better than the existing methods for many optimisation problems involving dimensions up to 5000.
\end{abstract}

\section{Introduction}
Bayesian optimization (BO) offers an efficient solution to find the global optimum of expensive black-box functions, a problem that is all pervasive in real-world experimental design applications. However, the scalability of BO is particularly compromised in high dimensions ($>15$ dimensions).

The main difficulty that a BO algorithm faces in high dimensions is that at each iteration, it needs to find the global maximum of a surrogate function called the  \emph{acquisition function} in order to suggest the next function evaluation point. The acquisition function balances two conflicting requirements: evaluating the function at a location where the function may peak as indicated by knowledge collected so far (exploitation) and evaluating the function at a location to reduce our uncertaiity about the function (exploration). The acquisition function itself is a non-convex optimisation problem in the original search space. With growing dimensions, any fixed computational budget for this optimisation becomes quickly insufficient  leading to inaccurate solutions. This inaccuracy adversely affects both the convergence and the efficiency of the BO algorithm.

In order to make the problem scalable, most of the current methods make restrictive structural assumptions such as the function having an effective low-dimensional subspace \cite{Wang13,Djolonga13,Garnett14,ErikssonDLBW18,nayebi19a,Zhang19}, or being decomposable in subsets of dimensions \cite{Kandasamy15,li16,rolland18a,Mutny18,HoangHOL18}. Through these assumptions, the acquisition function becomes easier to optimise and the global optimum can be found. However, such assumptions are rather strong and since we are dealing with unknown function, we have no way of knowing if such assumptions hold in reality.

Without these assumptions, the high-dimensional BO problem is more challenging. There have been some recent attempts to develop scalable BO methods. For example, \cite{Rana17} use an elastic Gaussian process model to reduce the "flatness" of the acquistion function in high dimensions and thus improve the solution. Although this method helps obtains improved optimum of the acquistion function, it still does not provide any way to scale BO to high dimensions. \cite{OhGW18} devise a method under the assumption that the solution does not lie at the boundary of the search space and thus drive the BO search towards the interior of the search space using polar co-ordinate transformations. Thus, the convergence analysis of \cite{OhGW18} depends on whether their assumptions holds. Some other methods are based on subspaces \cite{Qian16,Li17,Johannes19}, which are more amenable to convergence analysis. Among them, LineBO \cite{Johannes19} is the first to provide a complete analysis. In LineBO, the maximisation of the acquisition function is performed on a one-dimensional random subspace. Although this solution is computationally effective, the regret bound of LineBO does not scale well with the search space dimension and gets increasingly worse compared to the regret of the GP-UCB \cite{Srinivas12}. \emph{Thus we still need a method that is both computationally efficient and offers optimal sample efficiency in high dimensions in a flexible manner}.

We propose a novel algorithm for high-dimensional BO without making any restrictive assumptions on the function structure and show that our algorithm attains better regret than previous methods. The key insight of our method is that instead of maximizing the acquisition function on the full search space that is computationally prohibitive in high dimensions, we will maximize it on a restricted space that consist of multiple low-dimensional subspaces. This approach has several benefits. If the acquisition function is maximised in low-dimensional subspaces then computations involved in finding the solution of acquisition function optimisation are practically feasible while multiple subspaces can still cover the search space well if they are chosen in a principled manner. Further, this method allows us to theoretically derive a cumulative regret as a function of the number of subspaces. The crucial advantage is that we can \emph{trade between the computations and the convergence rate} of the algorithm while still maintaining efficient convergence (i.e. a sublinear growth of cumulative regret). Our contributions in this paper are:
\begin{itemize}
\item A novel BO algorithm based on multiple random subspaces that offers the flexibility to trade between the computations and the convergence rate. To do this, we propose to decompose the original search space into multiple lower-dimensional spaces via a randomisation technique.
\item We derive an upper bound on the cumulative regret for our algorithm and theoretically show that it has a sublinear growth rate and further, larger the number of subspaces, tighter the cumulative regret. Further, when the number of subspaces is large enough, the cumulative regret is at most of order $\mathcal{O}^{*}(\sqrt{T\gamma_T})$ as opposed to the regret bound of GP-UCB $\mathcal{O}^{*}(\sqrt{DT\gamma_T})$. The regret bound of our algorithm is tighter by a factor of $\sqrt{D}$. In some situations (detailed in the paper), this improvement is possible with fewer computations than GP-UCB - a double advantage.
\item We also study a special case when the objective function has an effective low dimensional subspace as assumed by many previous methods e.g. REMBO \cite{Wang13}. We show that our algorithm automatically benefits from this low dimensional structure with the tighter bound.
\item We extensively evaluate our method using a variety of optimisation tasks including optimisation of several benchmark functions as well as learning the parameters of a moderately sized neural network and a classification model based on non-convex ramp loss. We compare our method with several existing high dimensional BO algorithms and demonstrate that given equal computational budget, the sample efficiency of our method is consistently better than that of the existing methods.
\end{itemize}
\section{Preliminaries}
We consider the global maximization problem of the form
\begin{eqnarray}
\mathbf{x}^* = \text{argmax}_{\mathbf{x} \in \mathcal X}f(\mathbf{x})
\end{eqnarray}
in a compact search space $\mathcal X = [-1,1]^D$. In this paper we are especially concerned about problems with high values of $D$. We consider functions $f$ that are blackbox and expensive to evaluate, and our goal is to find the optimum in a minimal number of samples. We further assume we only have access to noisy evaluations of $f$ in the form $u = f(\mathbf{x}) + \epsilon$ where the noise $\epsilon \sim \mathcal N(0, \sigma^2)$ is i.i.d. Gaussian.
\vspace{-1em}
\subsection{Bayesian Optimization}
BO offers a principal framework to approach the global optimisation problem. The standard BO routine consists of two key steps: estimating the black-box function from function evaluation data and maximizing an acquisition function to suggest next function evaluation point balancing exploration and exploitation. Gaussian process (GP) \cite{Rasmussen05} is a popular choice for the first step due to its tractability for posterior and predictive distributions: $f(\mathbf{x}) = \mathcal {GP}(m(\mathbf{x}), k(\mathbf{x}, \mathbf{x}'))$,
where $m(\mathbf{x})$ and $k(\mathbf{x},\mathbf{x}')$ are the mean and the covariance (or kernel) functions. Popular covariance functions include linear kernel, squared exponential (SE) kernel, Matern kernel etc. The predictive mean and variance of Gaussian process is a Gaussian distribution. Given a set of observations $\mathcal D_{1:t} = \{\mathbf{x}_i, u_i\}_{i=1}^t$, the predictive distribution can be derived as:
$P(f_{t+1}| \mathcal D_{1:t}, \mathbf{x}) = \mathcal N(\mu_{t+1}(\mathbf{x}), \sigma_{t+1}^2(\mathbf{x}))$, where $\mu_{t+1}(\mathbf{x}) = \textbf{k}^T[\mathbf{K} + \sigma^2\textbf{I}]^{-1}\textbf{u} +  m(\mathbf{x})$ and $\sigma_{t+1}^2(\mathbf{x}) = k(\mathbf{x}, \textbf{x}) - \text{\textbf{k}}^{T}[\textbf{K} + \sigma^2\textbf{I}]^{-1}\textbf{k}$. In the above expression we define $\textbf{k }= [k(\mathbf{x}, \mathbf{x}_1), ..., k(\mathbf{x}, \mathbf{x}_t)]$, $\textbf{K }= [k(\mathbf{x}_i, \mathbf{x}_j)]_{1 \le i,j \le t}$ and $\textbf{u}=[u_1,\ldots,u_t]$.

The acquisition functions are designed to trade off between exploration of the search space and exploitation
of current promising region. Some examples of acquisition functions include Expected Improvement (EI) \cite{Mockus74} and GP-UCB \cite{Srinivas12}. A GP-UCB acquisition function at iteration $t+1$ is defined  as
\begin{eqnarray}
a_{t+1}(\mathbf{x}) =\mu_{t}(\mathbf{x}) + \sqrt{\beta_{t+1}}\sigma_{t}(\mathbf{x})
\label{eq:5}
\end{eqnarray}
where $\beta_{t+1}$ is a parameter to balance exploration and exploitation. There are guidelines \cite{Srinivas12} for setting $\beta_{t+1}$ to achieve sublinear regret.
\section{Proposed Method}
To solve the Bayesian optimization problem in high dimensions, the main difficulty is the prohibitive computational burden when maximising the acquisition function which requires solving a non-convex optimization problem in the same search space. Working with a small computational budget usually directly affects the quality of the point suggested by the acquisition step and consequently, many BO algorithms that require finding exactly the maximisers of the acquisition function (e.g. EI, GP-UCB) perform poorly in high dimensions. We propose a novel approach for this problem  as follows. Instead of maximising the acquisition function on the whole space, we perform it only on a discrete set of low-dimensional subspaces generated by random sampling. This approach brings several benefits. First, it solves the computation challenge effectively. Second, it provides a way to trade between the computations and the sample efficiency for convergence. Finally, our analysis suggests that when the number of subspaces are sufficiently high, our approach can simultaneously offer both lower computational requirement and better sample efficiency.

Before introducing our method, we formally define the low dimensional subspaces (see Definition \ref{de:2}) that will be used in our method.
\subsection{Subspace Description}
\begin{lemma}
Given a $d \in \{1, ..., D-1\}$, for any $x \in [-1,1]^D$, there exists a $y \in [-1,1]^d$ and a $z \in [-1, 1]^{D-d}$ such that $$\mathbf{x} = \mathbf{A}\mathbf{y} + \mathbf{B}\mathbf{z}$$ where matrix $\mathbf{A}= [\mathbf{0}^{d \times (D-d)}, \mathbf{I}^{d \times d}]^T$ and
matrix $\mathbf{B}= [\mathbf{I}^{(D-d) \times (D-d)}, \mathbf{0}^{(D-d) \times d}]^T$. \label{lem:1}
\end{lemma}
\begin{proof}
Given any $\mathbf{x} \in \mathcal X$, we denote the first $D-d$ elements of vector $\mathbf{x}$ by $[\mathbf{x}]_{1:D-d}$ and the last $d$ elements of vector $\mathbf{x}$ by $[\mathbf{x}]_{D-d + 1:D}$. We set $\mathbf{y} = [\mathbf{x}]_{D-d + 1:D} \in [-1,1]^d$ and $z = [\mathbf{x}]_{1:D-d} \in [-1,1]^{D-d}$. Thus, we have $\mathbf{x} = \mathbf{A}\mathbf{y} + \mathbf{B}\mathbf{z}$ with $\mathbf{y} \in [-1,1]^d$ and $\mathbf{z} \in [-1,1]^{D-d}$.
\end{proof}
The main idea is that we split the dimensions of any point $x$ in $D$ dimensional space into two groups: the first $D-d$ dimensions correspond to $z$ and the last $d$ dimensions correspond to $y$. Next, we define a set of subspaces based on this idea in the following form:
\begin{definition}[Embedding Subspace]
Given a $d \in \{1, ..., D-1\}$ and a vector $\mathbf{z} \in \mathcal Z = [-1, 1]^{D-d}$. We define an embedding subspace $\mathcal S(\mathbf{A}, \mathbf{z})$ as
\begin{eqnarray}
\mathcal S(\mathbf{A}, \mathbf{z}) = \{\mathbf{A}\mathbf{y} + \mathbf{B}\mathbf{z} | \mathbf{y} \in \mathcal{Y}=[-1,1]^d\}
\end{eqnarray}
where matrix $\mathbf{A}= [\mathbf{0}^{d \times (D-d)}, \mathbf{I}^{d \times d}]^T$ and matrix $\mathbf{B }= [\mathbf{I}^{(D-d) \times (D-d)}, \mathbf{0}^{(D-d) \times d}]^T$. \label{de:2}
\end{definition}
\subsection{The Proposed Algorithm}
Our algorithm (presented in Algorithm 1) closely follows the standard BO algorithm. The main  difference lies in the acquisition step that guides the selection of the next function evaluation point. We refer to our algorithm as \textbf{MS-UCB}.

Lemma \ref{lem:1} implies that there exists a $\mathbf{y}^* \in [-1,1]^d$ and $\mathbf{z}^* \in [-1, 1]^{D-d}$ such that
\begin{eqnarray}
\mathbf{x}^* = \mathbf{A}\mathbf{y}^* + \mathbf{B}\mathbf{z}^*
\end{eqnarray}
Therefore, we can preserve any optimum point $\mathbf{x}^*$ in the original space via $\mathbf{y}^*, \mathbf{z}^*$ in two lower-dimensional spaces (via dimension splitting). This observation gives rise to our new idea that instead of optimising $f$ on the original space, we can perform it in the lower dimensional spaces $\mathcal Y$ and $\mathcal Z$:
\begin{eqnarray}
\mathbf{y}^*, \mathbf{z}^* = \text{argmax}_{\mathbf{y} \in \mathcal Y, \mathbf{z} \in \mathcal Z}f(\mathbf{A}\mathbf{y} + \mathbf{B}\mathbf{z})
\label{eq:105}
\end{eqnarray}
where we denote the range $[-1,1]^d$ by $\mathcal Y$ and the range $[-1, 1]^{D-d}$ by $\mathcal Z$. Note that this is different from \cite{Wang13,Djolonga13} where they assume an effective subspace on $f$. In that case, using a low-dimensional subspace can solve the BO problem. In our problem without any restrictive assumption on $f$, we need to use two low-dimensional subspaces to preserve an optimum. As a result, when Eq (\ref{eq:105}) is established, maximising the acquisition function $a_t$ will be performed in spaces $\mathcal Y$ and $\mathcal Z$ as
\begin{eqnarray}
\mathbf{y}_t, \mathbf{z}_t = \text{argmax}_{\mathbf{y} \in \mathcal Y, \mathbf{z} \in \mathcal Z}a_t(\mathbf{A}\mathbf{y} + \mathbf{B}\mathbf{z})
\end{eqnarray}
In high dimensions where we may set $d<<D$, the $(D-d)$-dimensional space $\mathcal{Z}$ would itself be a high-dimensional space. Therefore, maximising the acquisition function in this space is still computationally expensive. To deal with this problem, we work with a finite set of samples $\mathbf{z} \in \mathcal Z$ and maximise the acquisition function only on this finite set instead of the whole space $\mathcal Z$. Let $Z_t$ be the set of all $\mathbf{z}$ generated up to iteration $t$. With this modification, our acquisition step becomes:
\begin{eqnarray}
\mathbf{y}_t, \mathbf{z}_t = \text{argmax}_{\mathbf{y} \in \mathcal Y, \mathbf{z} \in \mathcal Z_t}a_t(\mathbf{A}\mathbf{y} + \mathbf{B}\mathbf{z})
\label{eq:10}
\end{eqnarray}
Once $\mathbf{z}$ is sampled, a subspace $\mathcal S(\mathbf{A}, \mathbf{z})$ as in Definition \ref{de:2} is generated. Let $\mathcal X_t \triangleq \{\mathcal{S}(\mathbf{A}, \mathbf{z}^i)\mid\mathbf{z}^i \in Z_t\}$.  The Eq (\ref{eq:10}) can then be re-written as
\begin{eqnarray}
\mathbf{x}_t = \text{argmax}_{\mathbf{x} \in \mathcal X_t}a_t(\mathbf{x})
\label{eq:12}
\end{eqnarray}
By Definition \ref{de:2}, we can see that for any subspace $\mathcal S(\mathbf{A}, \mathbf{z}^i)$, $\mathcal S(\mathbf{A}, \mathbf{z}^i) \subset \mathcal X$. Thus, the suggested point $x_t$ is always within $\mathcal X$. This is a benefit of our dimension splitting based subspace projection as opposed to complex corrections required for previous methods (\cite{Wang13,Qian16}). We note that for our method $a_t(\mathbf{x}) = \mu_{t-1}(\mathbf{x}) + \sqrt{\beta_{t}}\sigma_{t-1}(\mathbf{x})$ where $\beta_t = 2log(\frac{\pi^2t^2}{\delta}) +  2dlog(2bd\sqrt{log(\frac{6Da}{\delta})}t^2)$. The $\beta_t$ of GP-UCB depends linearly on $D$ whereas the dependence of our $\beta_t$ on $D$ is only $\sqrt{\text{log}D}$.

A simple way to maximise $a_t$ on $\mathcal X_t$ is to maximise $a_t$ on each subspace $\mathcal S(\mathbf{A}, \mathbf{z}^i_t)$ separately and thus the returned value $\mathbf{x}_t$ is the maximizer on the all subspaces at iteration $t$.
For a small $d$, $\mathcal S(\mathbf{A}, \mathbf{z}^i)$ is a low-dimensional subspace of $\mathcal X$. Thus, maximizing the acquisition function on such a subspace is computationally cheaper.

Importantly, we can show that maximising $a_t$ on a discrete set of subspaces can result in low regret by proposing a strategy to choose set $Z_t$. At iteration $t$, we uniformly randomly draw $N_t$ samples of $z$, $ \{ \mathbf{z}_t^i \in \mathcal Z \mid i = 1,...,N_t \} $ and then construct $Z_t$ as $Z_t = Z_{t-1} \cup \{\mathbf{z}^1_{t}, ..., \mathbf{z}^{N_t}_t\}$. We choose $N_t = N_0 t^{\alpha}$, where $N_0, \alpha \in \mathbb{N}$, $N_0 \ge 1$ and $\alpha \ge 0$.
 \begin{algorithm}[tb]
\caption{MS-UCB Algorithm}
\label{alg:alg}
\textbf{Input}: Input space $\mathcal X= [-1, 1]^D$; a low dimension $1\le d < D$, $Z_0 = \emptyset$, the parameters $N_0$ and $\alpha$.
\begin{algorithmic}[1] 
\State Sample initial points to construct $\mathcal D_{0}$.
\State Build a Gaussian process using $\mathcal D_{0}$.
\For {$t = 1, 2, ...,T$}
    \State Sample uniformly at random $N_t$ values of $\mathbf{z}^i_t \in [-1, 1]^{D-d}$, where $1 \le i \le N_t=N_0t^{\alpha}$.
    \State Update $Z_t = Z_{t-1} \cup \{\mathbf{z}^1_{t}, ..., \mathbf{z}^{N_t}_t\}$.
    \State Maximise acquisition function to obtain $\mathbf{x}_t$ by  	following Eq (\ref{eq:12}).
    \State Sample $u_t = f(\mathbf{x}_t) + \epsilon_t$.
    \State Augment the data $\mathcal D_{t} = \{\mathcal D_{t-1}, (\mathbf{x}_{t}, u_t)\}$.
    \State Update the Gaussian process using $\mathcal D_{t}$.
\EndFor

\end{algorithmic}
\end{algorithm}
\section{Convergence Analysis}
In this section, we analyse the convergence of our proposed MS-UCB algorithm. For this, we use the regret, which tells us how much better we could have done in iteration $t$ had we known $\mathbf{x}^*$, formally $r_t = f(\mathbf{x}^*) - f(\mathbf{x}_t)$. In many applications, such as recommender systems, robotic control, etc, we care about the quality of the points chosen at every iteration $t$, and thus a natural quantity to consider is the cumulative regret that is defined as $R_T = \sum_{1 \le t \le T}r_t$ , the sum of regrets incurred over a horizon of $T$ iterations. If we can show that the cumulative regret is sublinear for a given algorithm, then $lim_{T \rightarrow \infty}R_T/T = 0$, meaning the algorithm efficiently converges to the optimum.

To derive a cumulative regret $R_t = \sum_{t=1}^{T} r_t$, we will seek to bound $r_t = f(\mathbf{x}^*)- f(\mathbf{x}_t)$ for any $t$. At each iteration $t$, we denote by $\mathcal S_t^0 \triangleq \{\mathbf{A}\mathbf{y}^* + \mathbf{B}\mathbf{z}_i\}_{\mathbf{z}_i \in Z_t}$. Let $\mathbf{z}^*_t \triangleq \text{argmin}_{\mathbf{z} \in \mathcal Z_t}||\mathbf{z} -\mathbf{z}^*||_1$ and $f^{max}_{\mathcal S_t^0} \triangleq f(\mathbf{A}\mathbf{y}^* + \mathbf{B}\mathbf{z}^*_t)$.
To obtain a bound on $r_t$, we write it as
\begin{eqnarray}
r_t & = & f(\mathbf{x}^*)- f(\mathbf{x}_t) \\
& = & \underbrace{f(\mathbf{x}^*) - f^{max}_{\mathcal S_t^0}}_\text{Term 1} +  \underbrace{f^{max}_{\mathcal S_t^0}}_\text{Term 2} - \underbrace{f(\mathbf{x}_t)}_\text{Term 3}
\end{eqnarray}
Term 1 will be bounded from the result of Lemma \ref{lem:5.1}, Term 2 will be bounded from the result of Lemma \ref{lem:5.3} and Term 3 will bounded from the result of Lemma \ref{lem:5.4}. Before proceeding to the proofs, we need to make the following assumption.
\begin{assumption}[Gradients of GP Sample Paths \cite{ghosal2006}]
Let $f \sim \mathcal{GP}(\mathbf{0}, k)$ and $k$ is a stationary kernel. The partial derivatives of $f$ satisfies the following condition. There exist constants $a, b > 0$ such that
$$\mathbb{P}[sup_{\mathbf{x} \in \mathcal X}|\diffp{f}{{x_{i}}}| > L] \le ae^{-(L/b)^2}$$
for all $L > 0$ and for all $i \in \{1,2,..., D\}$
\label{as:1}
\end{assumption}
Similar to Lemma 5.5 of \cite{Srinivas12}, we have the following bound on the actual function observations.
\begin{lemma}[Bounding Term 3]
Pick a $\delta \in (0,1)$ and set $\beta_t^0 = 2log(\pi^2t^2/(6\delta))$. Then we have
\begin{eqnarray}
f(\mathbf{x}_t) \ge \mu_{t-1}(\mathbf{x}_t) - \sqrt{\beta_t^0} \sigma_{t-1}(\mathbf{x}_t)
\end{eqnarray}
holds with probability $\ge 1 - \delta$.
\label{lem:5.4}
\end{lemma}
\begin{lemma}[Bounding Term 2]
Pick a $\delta \in (0,1)$ and set $\beta_{t}^1 = 2log(\frac{\pi^2t^2}{3\delta}) +  2dlog(2bd\sqrt{log(\frac{2Da}{\delta})}t^2)$. Then, under Assumption \ref{as:1} there exists a $\mathbf{x}' \in \mathcal S(\mathbf{A}, \mathbf{z}^*_t)$ such that
\begin{eqnarray}
f^{max}_{\mathcal S_t^0} \le  \mu_{t-1}(\mathbf{x}') + \sqrt{\beta_t^1}\sigma_{t-1}(\mathbf{x}')  + \frac{1}{t^2}
\end{eqnarray}
holds with probability $\ge 1 - \delta$.
\label{lem:5.3}
\end{lemma}
\begin{lemma}[Bounding Term 1]
Pick a $\delta \in (0,1)$ and set $v_0 = 2b\sqrt{log(\frac{2Da}{\delta})}(\Gamma(D-d +1))^{\frac{1}{D-d}}$, where $\Gamma(D-d +1) = (D-d)!$. With probability at least $1- \delta$, we have
\begin{eqnarray}
f(\mathbf{x}^*)- f_{\mathcal S_t^0}^{max} \le v_0(\frac{1}{|Z_t|}log(\frac{2}{\delta}))^{\frac{1}{D-d}}.
\end{eqnarray}
\label{lem:5.1}
\end{lemma}
All proofs are provided in details in Supplementary Material. Different from confidence bound techniques that are used to bound Term 2 and Term 3, we use the randomisation techniques on sampled set $Z_t$ and a result from \cite{Xian05} to bound Term 1. Combining the results from Lemmas \ref{lem:5.4}, \ref{lem:5.3} and \ref{lem:5.1} we obtain a bound on $r_t$ and then sum it over  iteration 1 to $T$ to obtain the following theorem providing an upper bound on the cumulative regret $R_T$. The notation $\mathcal{O}^*$ is a variant of $\mathcal{O}$, where
log factors are suppressed.
\begin{theorem}
Pick a $\delta \in (0,1)$. Then, the cumulative regret of the proposed MS-UCB algorithm is bounded as
\begin{itemize}
  \item $R_T \le \mathcal{O}^*(\sqrt{dT\gamma_T} + D)$ if $\alpha \ge D-d-1$.
  \item $R_T \le \mathcal{O}^*(\sqrt{dT\gamma_T} + \frac{D^2-Dd}{D-d-\alpha -1}T^{1 -\frac{\alpha +1}{D-d}})$ if $0 \le \alpha < D-d -1$.
\end{itemize}
with probability greater than $1 -\delta$, where $\gamma_T$ is the
maximum information gain about the function $f$ from any set of observations of size $T$. $\alpha$ is related to the number of subspaces chosen as $N_t = N_0 t^{\alpha}$.
\label{theorem:1}
\end{theorem}
Next, we show that if the objective function has a low dimensional effective subspace meaning when there are directions in which the function is constant, our algorithm can benefit automatically from this structure. The following Theorem provides a regret bound for this scenario.
\begin{theorem}[For functions having an effective subspace]
Pick a $\delta \in (0,1)$. If there exists a linear subspace $\mathcal T \subset \mathbb{R}^D$ with $d_e$ dimensions such that $ \forall \mathbf{x} \in \mathbb{R}^D$, $f(\mathbf{x})= f(\mathbf{x}_{\top}) $ where $\mathbf{x}_{\top} \in \mathcal T$ is the orthogonal projection of $x$ onto $\mathcal T$, then
\begin{itemize}
  \item $R_T \le \mathcal{O}^*(\sqrt{dT\gamma_T} + d_e)$ if $\alpha \ge d_e - 1$.
  \item $R_T \le \mathcal{O}^*(\sqrt{dT\gamma_T} + \frac{d_e^2}{d_e-\alpha -1}T^{1 -(\alpha +1)/d_e})$ if $0 \le \alpha < d_e -1$.
\end{itemize}
with probability greater than $1 -\delta$, where $\gamma_T$ is the
maximum information gain about the function $f$ from any set of observations of size $T$.
\end{theorem}
\section{Discussion}
There are two important parameters in our method, (1) the dimension $d$ that is used for acquisition function maximisation and (2) parameter $\alpha$ that is related to the number of subspaces on which the acquisition function is maximised. In this section, we provide a discussion on these parameters. We also analyse the computational cost involved in maximising the acquisition function comparing it with that of GP-UCB.
\vspace{-1em}
\paragraph{On Dimension $d$.} The proposed MS-UCB algorithm is applied for $1 \le d < D$,  We consider two extreme cases where $d = 0$ or $d = D$. \begin{itemize}
  \item For $d =0$, our Definition \ref{de:2} is still valid however, in this case, the space $\mathcal Y$ is degenerated to 0. It follows that $\mathcal S(A,\mathbf{z}) = \{\mathbf{z}\}$ that means that the subspace becomes a unique point. The acquisition step  becomes finding the maximiser of the acquisition function of a set of sampled points. We obtain the same result as in Theorem \ref{theorem:1} for $d =0$ with a slight modification of $\beta_t = 4log(\pi^2t^2/2\delta)$ for Algorithm 1.
  \item For $d = D$, the space $\mathcal Z$ is degenerated into 0 and space $\mathcal Y$ becomes space $\mathcal X$. In this case, the idea of using sampling on space $\mathcal Z$ is not being utilised and the algorithm reduces to the standard GP-UCB algorithm working directly on the original high-dimensional space $\mathcal X$.
\end{itemize}
\paragraph{On the Number of Subspaces.} In Theorem \ref{theorem:1}, the parameter $N_0$ does not affect the convergence rate of $R_T$. We thus only discuss the different cases of $\alpha$. When $\alpha = 0$, Theorem \ref{theorem:1} says that $R_T \le \mathcal{O}^*(\sqrt{dT\gamma_T} + \frac{D^2-Dd}{D-d-1}T^{1 - 1/(D-d)})$. Although the regret growth now has an additional term $T^{1 -1/(D-d)}$, but even in this case our algorithm has a sublinear cumulative regret. As expected, larger the value of $\alpha$, tighter the cumulative regret until $\alpha \ge D-d -1$, after which the cumulative regret no longer depends on $\alpha$ as the term $T^{1 -(\alpha+1)/(D-d)}$ gets dominated by $\sqrt{dT\gamma_T}$. In this case, $R_T \le \mathcal{O}^*(\sqrt{dT\gamma_T})$ for a large enough $T$. Comparison with regret bound of GP-UCB $\mathcal{O}^{*}(\sqrt{DT\gamma_T})$, our algorithm's bound is tighter by the factor $\sqrt{D}$. \emph{To our best knowledge, ours is the first algorithm that obtains a cumulative regret bound better than GP-UCB's for any $D$ being the dimensional number of input space under the Assumption 1}.
\paragraph{On the Computation Complexity.} As mentioned in \cite{Kandasamy15}, in any grid or branch and bound methods, maximising a function to within $\zeta$ accuracy, requires $\mathcal{O}(\zeta^{-D})$ calls to $a_t$. Since we need to solve $|Z_t|$ $d$ dimensional optimization problems for acquisition functions, it requires only $\mathcal{O}(|Z_t|\zeta^{-d})$ calls.

In our algorithm, we consider $N_t = N_0t^{\alpha}$, and thus $|Z_t| = \sum_{j =1}^{t}N_0 j^{\alpha}$, which can be bounded by $\frac{(t+1)^{\alpha +1}-1}{\alpha +1}$ using the results from \cite{Chlebus2009}. The largest computation is at iteration $T$ where $|Z_T| = \sum_{j =1}^{T}N_0 j^{\alpha} < N_0(T+1)^{\alpha +1}/(\alpha +1)$. To have reduced computations in maximising the acquisition function, we should set $\alpha$ so that $(N_0(T+1)^{\alpha +1}/(\alpha +1)) \zeta^{-d} < \zeta^{-D}$. If we choose $\zeta = \frac{1}{(T+1)^2}$ and $N_0 = 1$ then by choosing $\alpha < 2(D-d) -1$, the condition is satisfied. Thus, combining with Theorem \ref{theorem:1}, we can say, if there exists $$ D-d -1 \le \alpha < 2(D-d)-1$$ then our algorithm can obtain both a cumulative regret $\mathcal{O}^*(\sqrt{dT\gamma_T})$ that is tighter than GP-UCB's and a computational cost that is cheaper than GP-UCB's when maximising the acquisition function with $\frac{1}{(T+1)^2}$- accuracy.
\section{Experiments}
\begin{figure*}[t]
\centering
\subfloat{\includegraphics[scale=0.5,width=.25\textwidth]{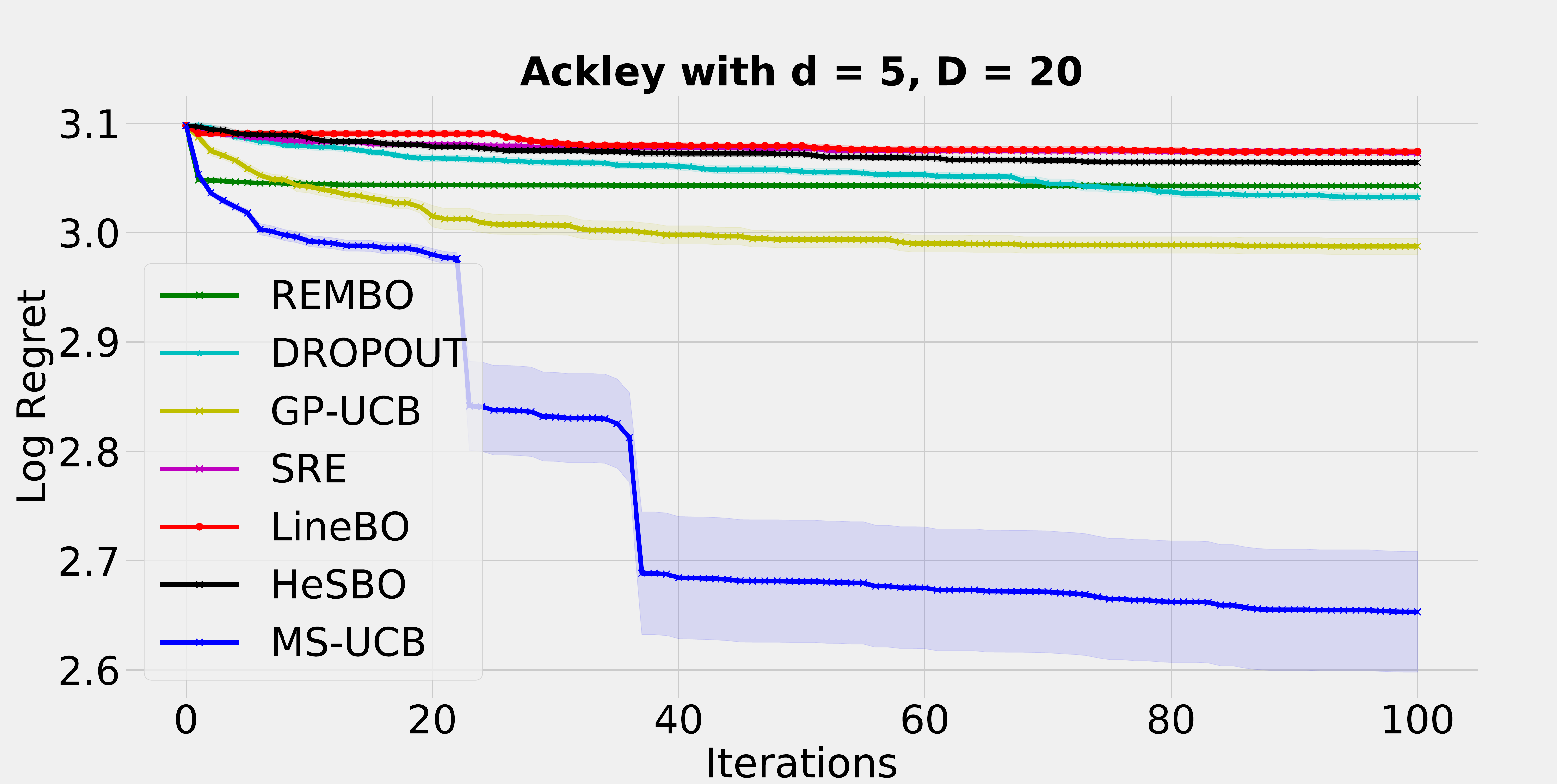}
}\hfill
\subfloat{\includegraphics[scale=0.5,width=.25\textwidth]{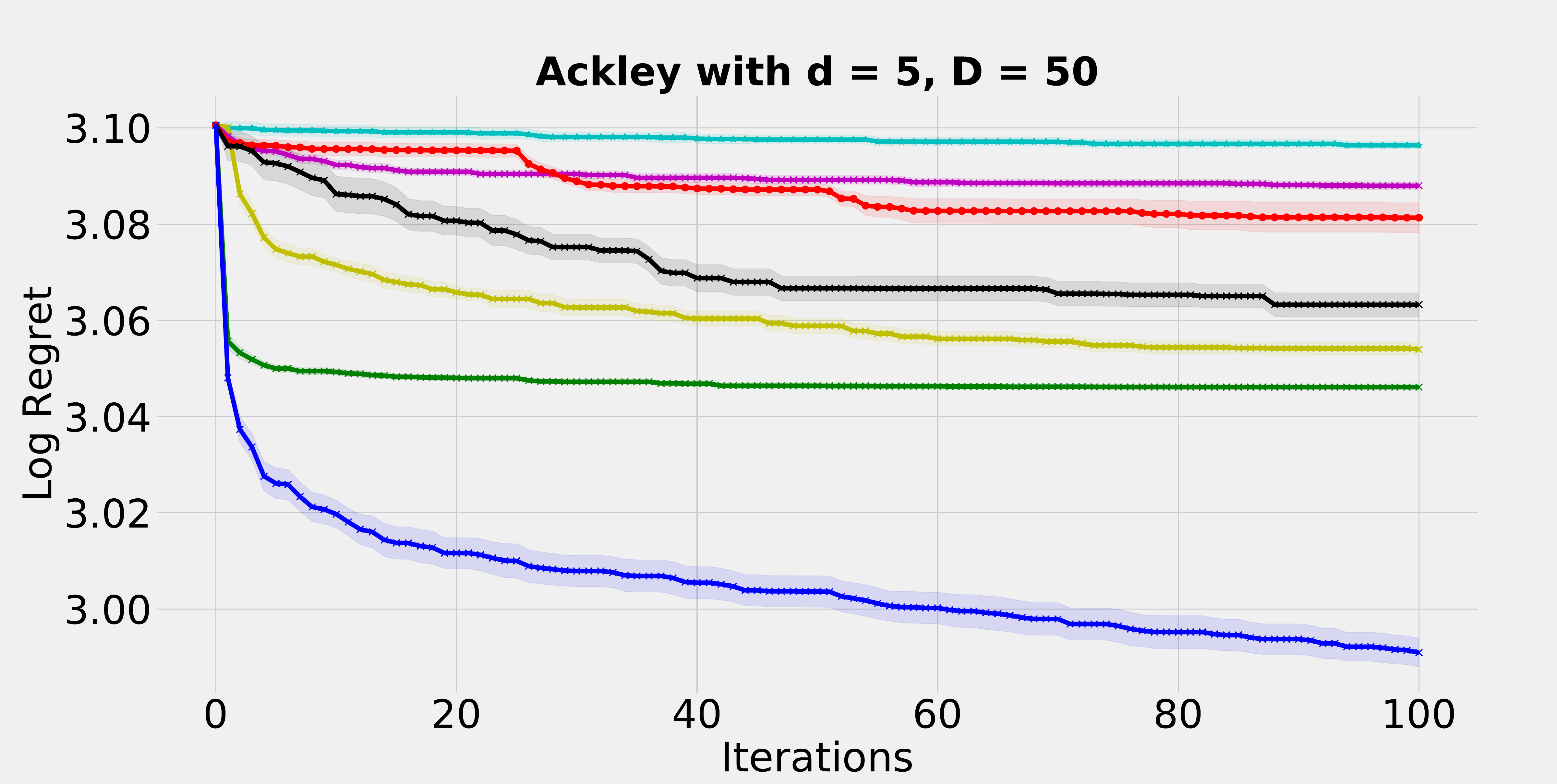}
}\hfill
\subfloat{\includegraphics[scale=0.5,width=.25\textwidth]{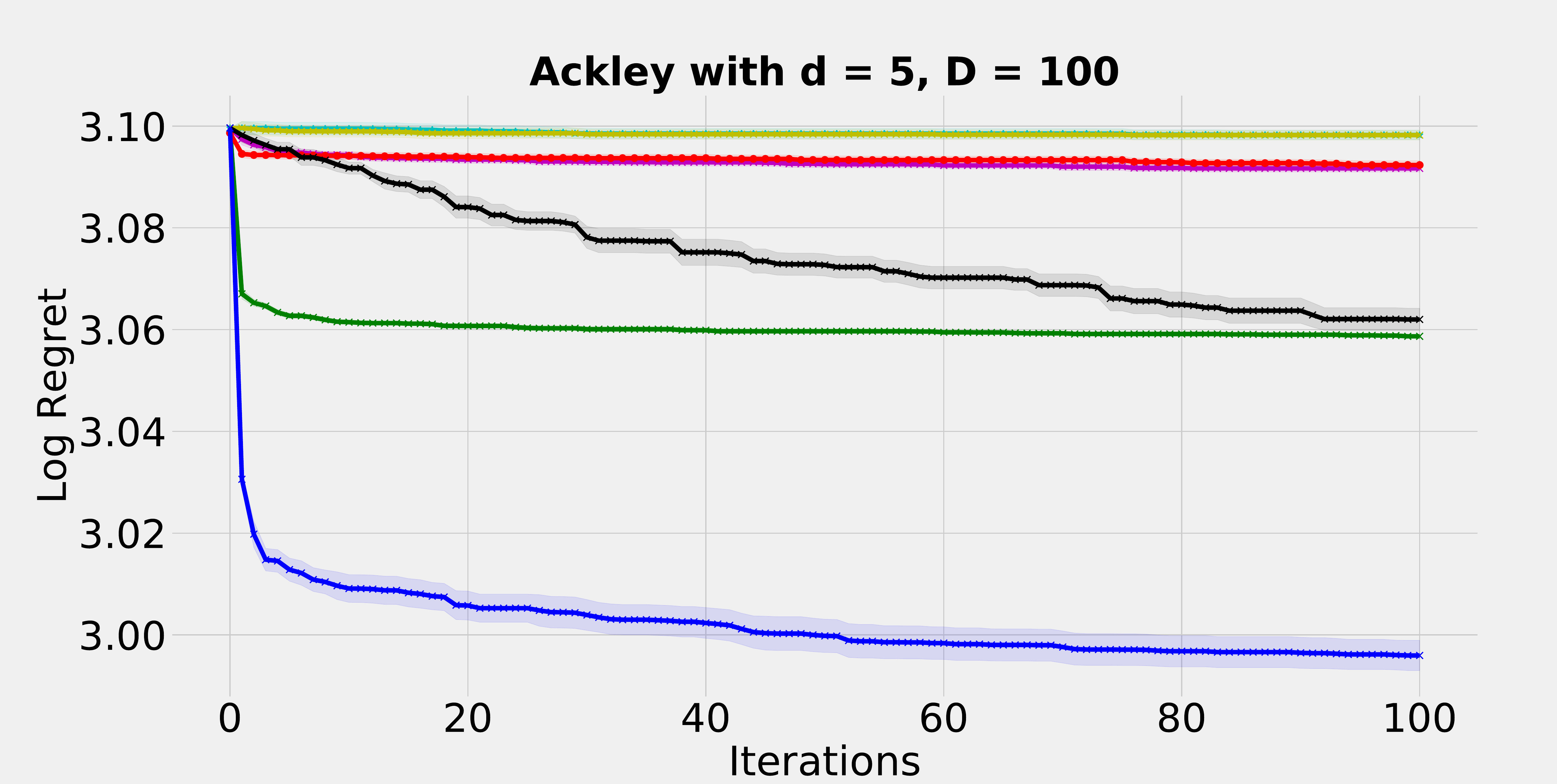}
}\\
\subfloat{\includegraphics[scale=0.5,width=.25\textwidth]{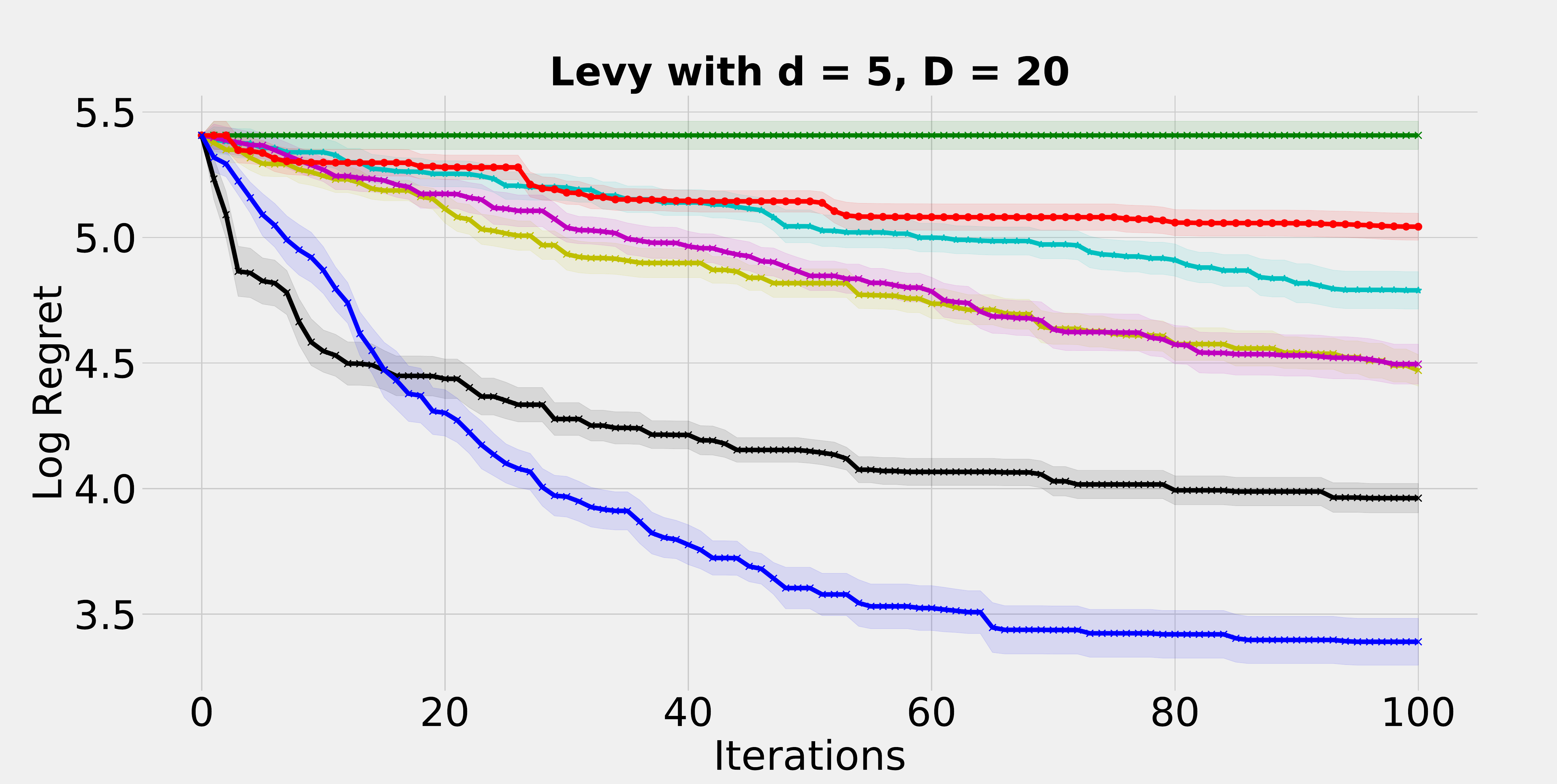}
}\hfill
\subfloat{\includegraphics[scale=0.5,width=.25\textwidth]{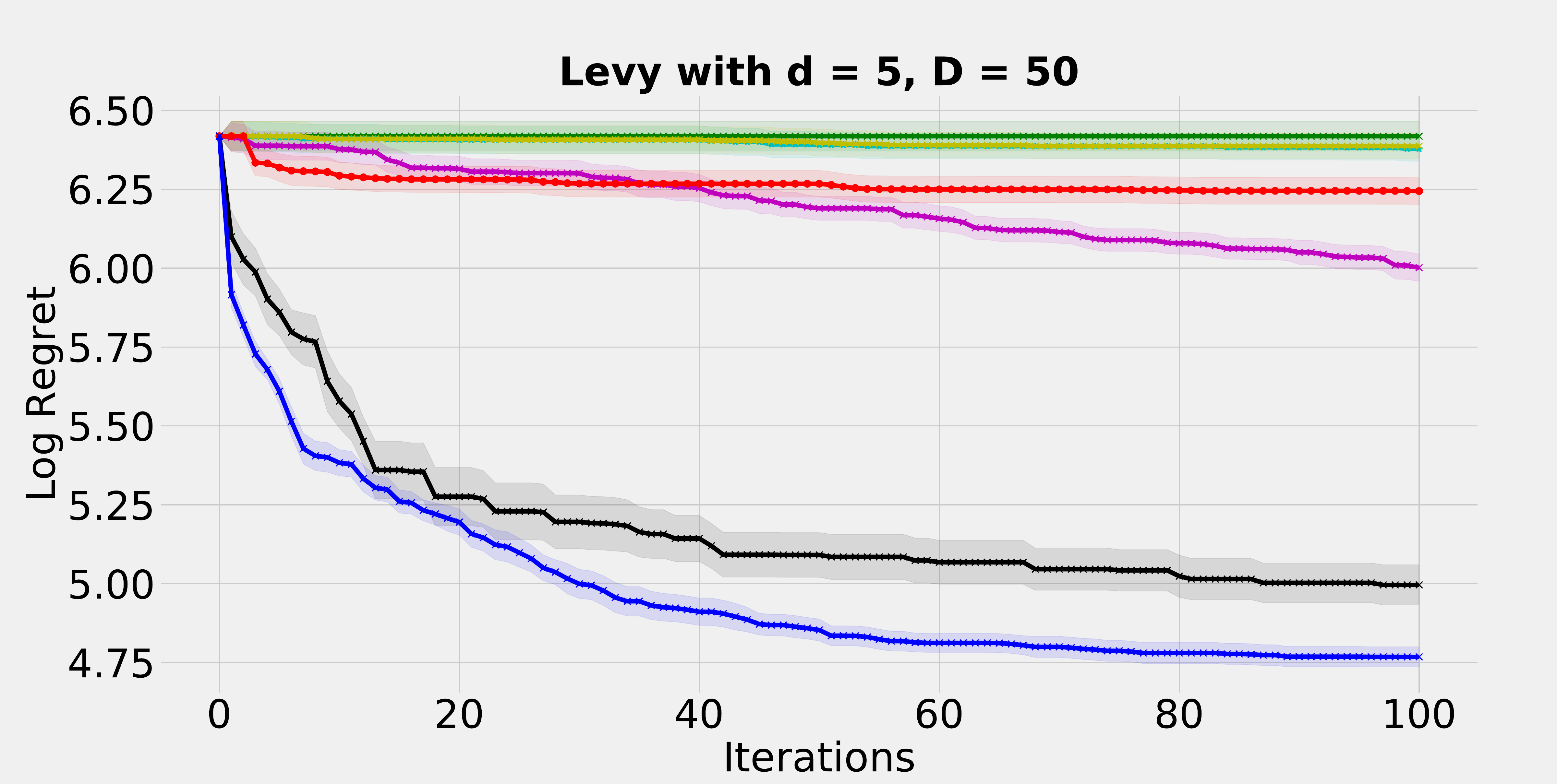}
}\hfill
\subfloat{\includegraphics[scale=0.5,width=.25\textwidth]{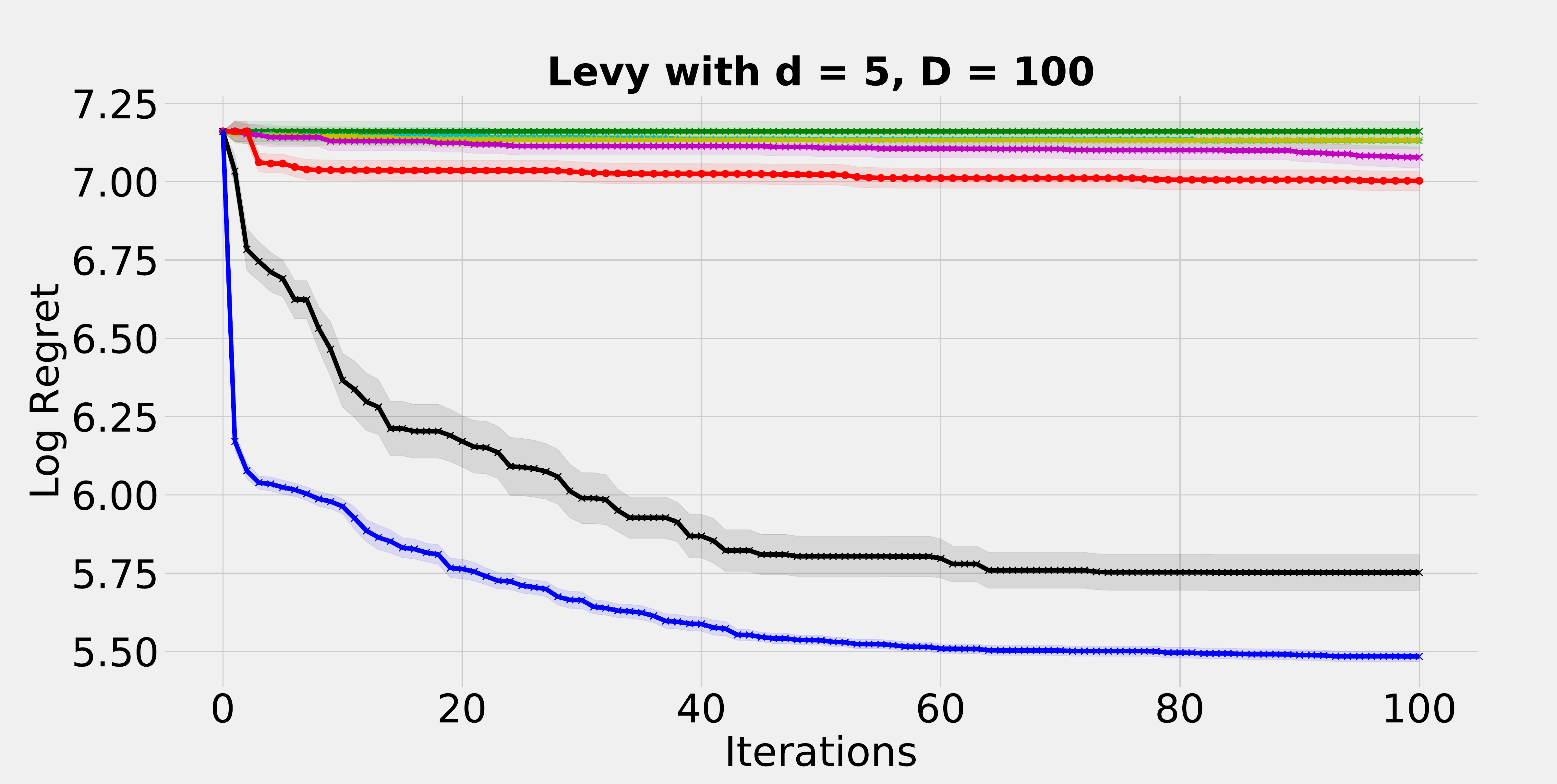}
}\\
\subfloat{\includegraphics[scale=0.5,width=.25\textwidth]{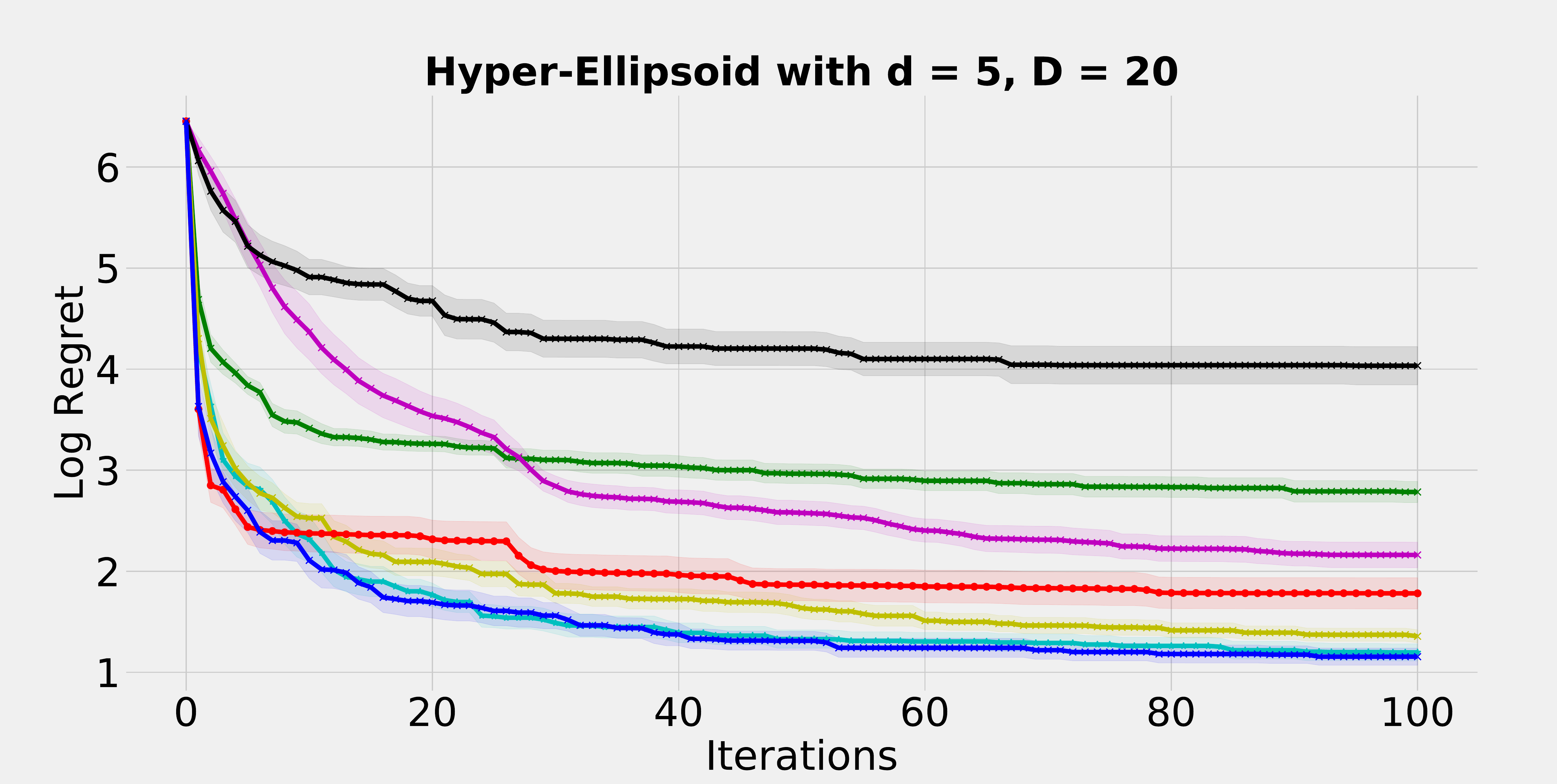}
}\hfill
\subfloat{\includegraphics[scale=0.5,width=.25\textwidth]{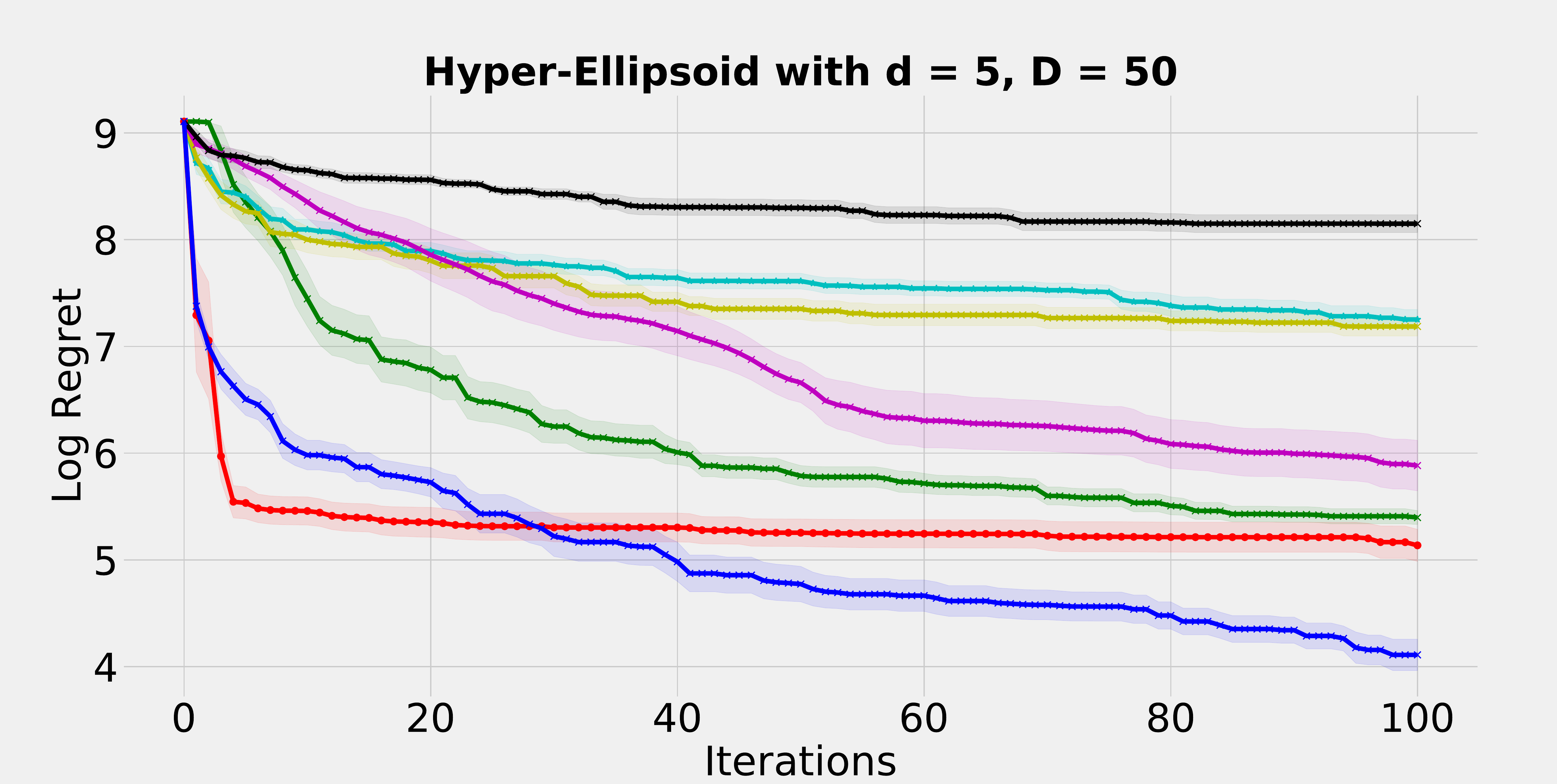}
}\hfill
\subfloat{\includegraphics[scale=0.5,width=.25\textwidth]{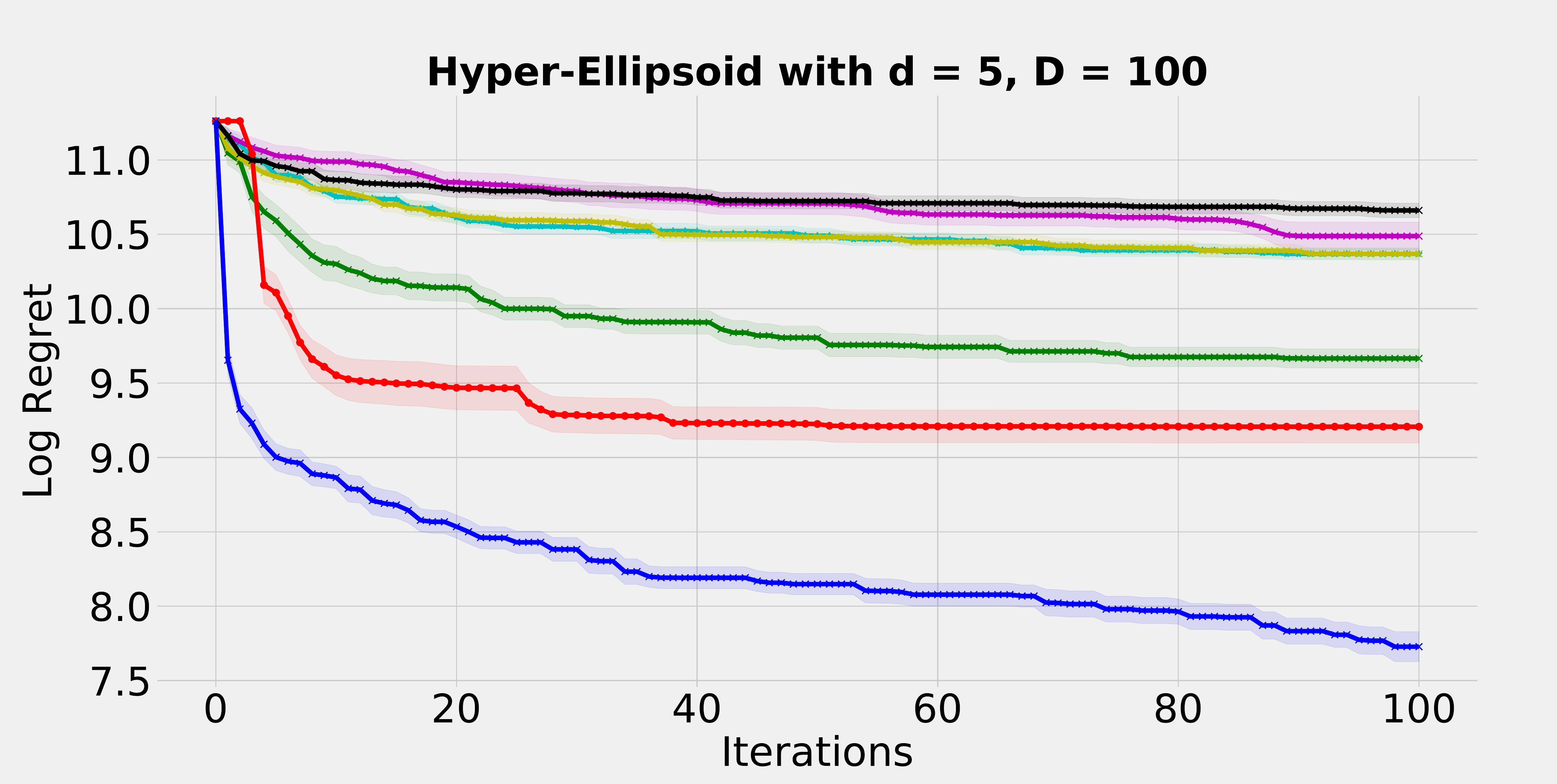}
}\\
\subfloat{\includegraphics[scale=0.5,width=.25\textwidth]{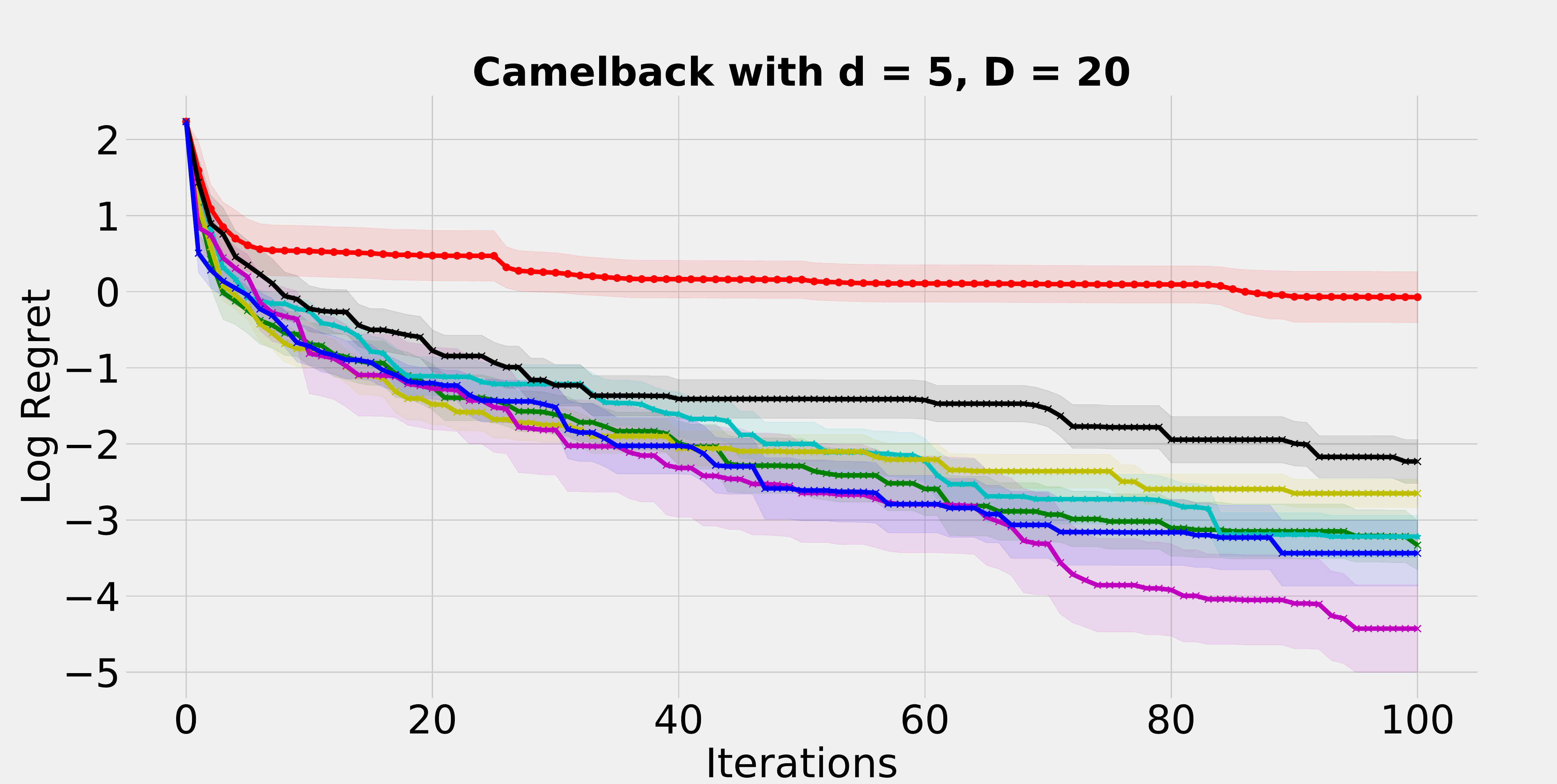}
}\hfill
\subfloat{\includegraphics[scale=0.5,width=.25\textwidth]{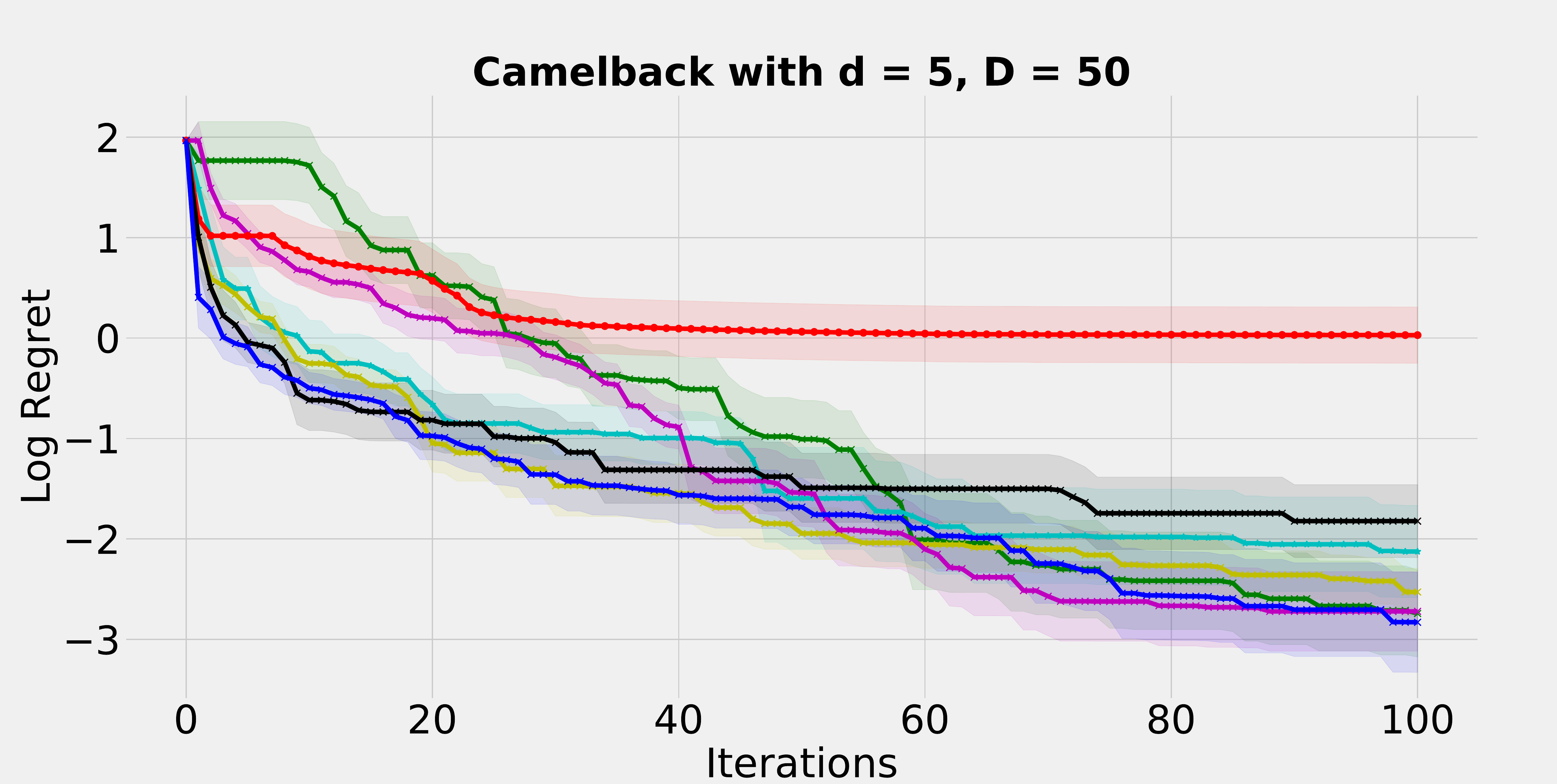}
}\hfill
\subfloat{\includegraphics[scale=0.5,width=.25\textwidth]{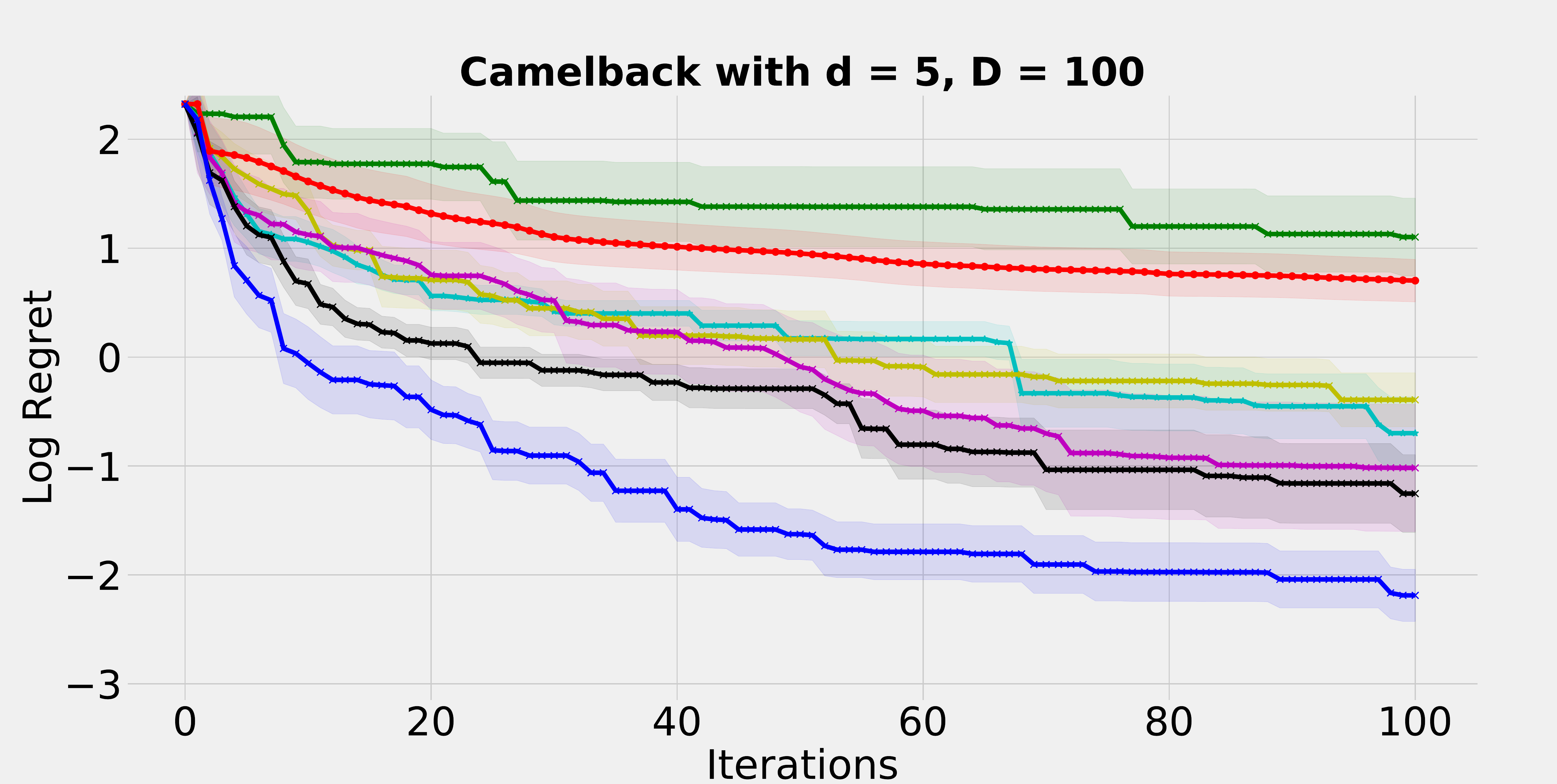}
}
\caption{Comparison of baselines and the proposed MS-UCB method on four standard functions for 20, 50 and 100 input dimensions. For all cases, we set $d = 5$ except LineBO and GP-UCB. The $y$-axis presents log distance to the true optimum (Smaller value is better).}
\label{freq}
\end{figure*}
To evaluate the performance of our MS-UCB, we have conducted a set of experiments involving optimization of four benchmark functions and two real applications. We compare our approach against six baselines: (1) Standard GP-UCB \cite{Srinivas12}, (2) DropoutUCB \cite{Li17}, (3) LineBO \cite{Johannes19} which restricts the search space to a one-dimensional subspace, (4) SRE \cite{Qian16} which uses sequential random embeddings several times sequentially, (5) REMBO \cite{Wang13}, and (6) HeSBO \cite{nayebi19a} which use hashing-enhanced embedded subspaces. Among these baselines, the first three baselines do not make assumptions on the structure of the objective function, SRE assumes a tiny effect for some of the dimensions i.e. $\epsilon$-bounded while REMBO and HeSBO assume a low effective dimensional structure of the function.

For all experiments, we scale the search space of objective functions to convert into $[-1,1]^D$. We implemented our proposed MS-UCB, LineBO, DropoutUCB and SRE in Python 3 using GPy. For all other algorithms we used the authors' reference implementations. For Gaussian process, we used Matern kernel and estimated the kernel hyper-parameters automatically from data. Each algorithm was randomly initialized with 20 points. To maximise the acquisition function, we used LBFGS-B algorithm with $10 \times D$ random starts.
\subsection{Optimization of Benchmark Functions}In this section, we test the algorithms on several optimization benchmark functions: Ackley, Levy, Hyper-Elippsoid and Camelback functions. For the first three functions we assume full dimensionality while for the two-dimensional Camelback function, we simulate a scenario so that the function has a low dimensional effective subspace. For this, we augment the Camelback function with auxiliary dimensions. The Ackley function is widely used for testing optimization algorithms. it is characterized by a nearly flat outer region while the Hyper-Elippsoid function is used to  demonstrate that our
algorithms can effectively work for functions with interacting variables.
We evaluate the progress of each algorithm using the log distance to the true optimum, that is, $\text{log}_{10}(f(x^*) - f(x_t))$ where $f(x_t)$ is the function value sampled at iteration $t$. For each test function, we repeat the experiments 30 times. We plot the mean and a confidence bound of one standard deviation across all the runs.
\paragraph{On Scalability:} We perform experiments to empirically assess the scalability of our proposed MS-USB method and the baselines and report the results in Figure 1. We choose $d =5$ for all methods except LineBO for which $d=1$ is a requirement and the GP-UCB which works directly in original $D$-dimensional space. We use $N_0 =1, \alpha = 0$ as parameters for our method. Recall that even with $\alpha = 0$, our method only has a sublinear cumulative regret growth. We study the cases with different input dimensions: $(D =20, 50, 100)$. We can see that GP-UCB performs poorly in most cases on all test functions. The poor performance of GP-UCB is partly due to inaccurate solution of acquisition function optimisation in high dimensions. On the other hand, our method does better than all the baselines scaling well with the dimensions. Our method overcomes the difficulty by optimizing the acquisition function only on a set of $d$-dimensional subspaces, and thus with a limited computation budget, the acquisition function optimisation is performed more accurately. Dropout and SRE that do not provide a vanishing regret scale poorly with high dimensions. LineBO performs poorly except for the Hyper-Ellipsoid function. For the Camelback function with just two effective dimensions, our method is still competitive to other methods that are designed to exploit such structure. Our method achieves a slightly better accuracy than SRE when $D = 50$ and the best one when $D =100$. Thus our proposed MS-UCB outperforms baselines and scales well in high dimensions.
\paragraph{On the number of subspaces:} The number of subspaces per iteration is a control parameter in our method. It is set via $N_0$ and in particular $\alpha$. In Theorem \ref{theorem:1}, we show that $N_0$ does not affect much the regret. We study the effect of these parameters creating four MS-UCB variants: $(N_0 = 1, \alpha = 0)$, $(N_0 = 10, \alpha = 0)$, $(N_0 = 1, \alpha = 1)$ and finally $(N_0 = 1, \alpha = 2)$. We fix $D =100$ and $d =5$.
\begin{figure}[t]
\centering
\subfloat[Hyper-Ellipsoid function]{\includegraphics[scale=1.0,width=.23\textwidth]{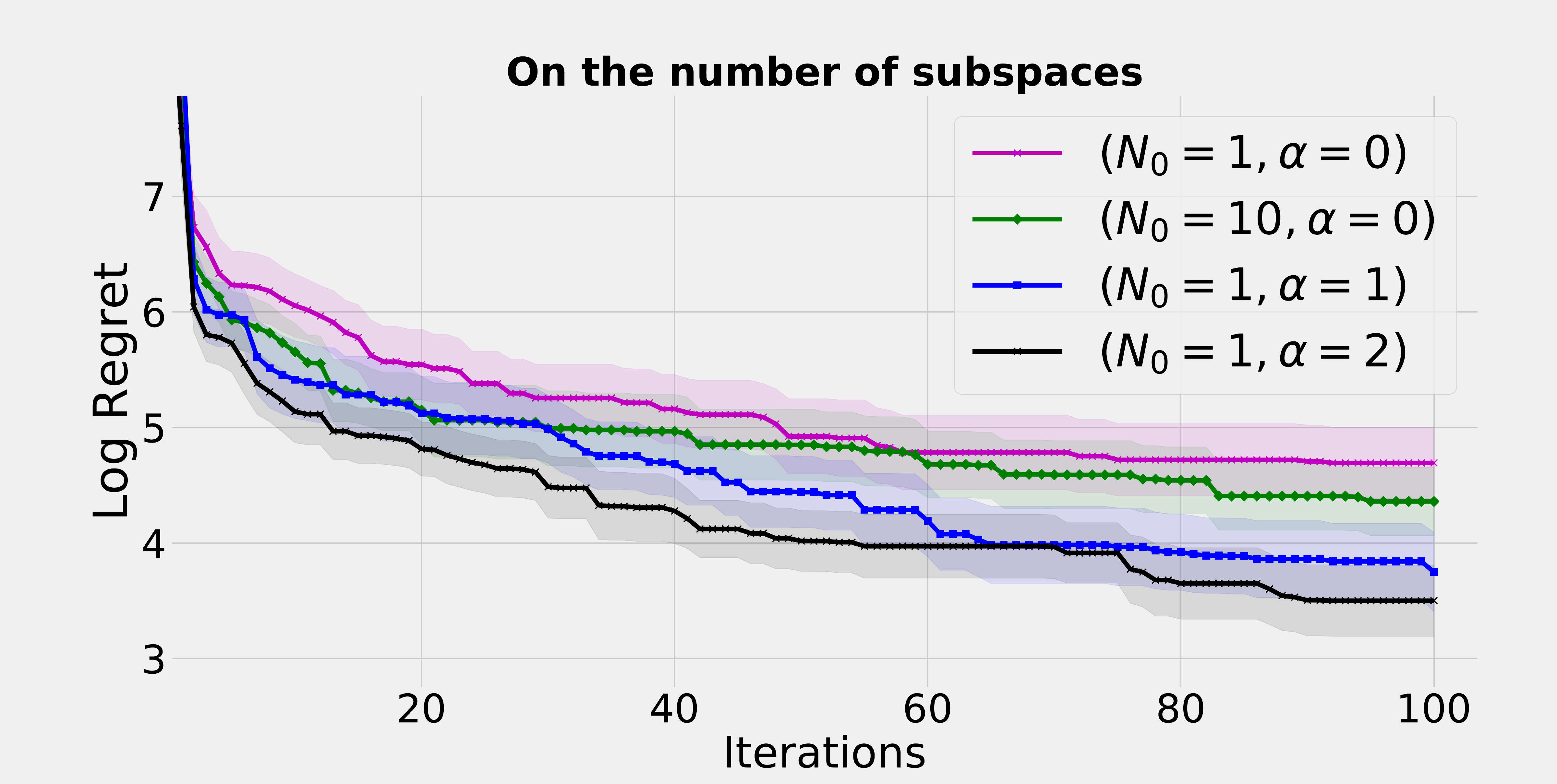}
}\hfill
\subfloat[Ackley function]{\includegraphics[scale=1.0,width=.23\textwidth]{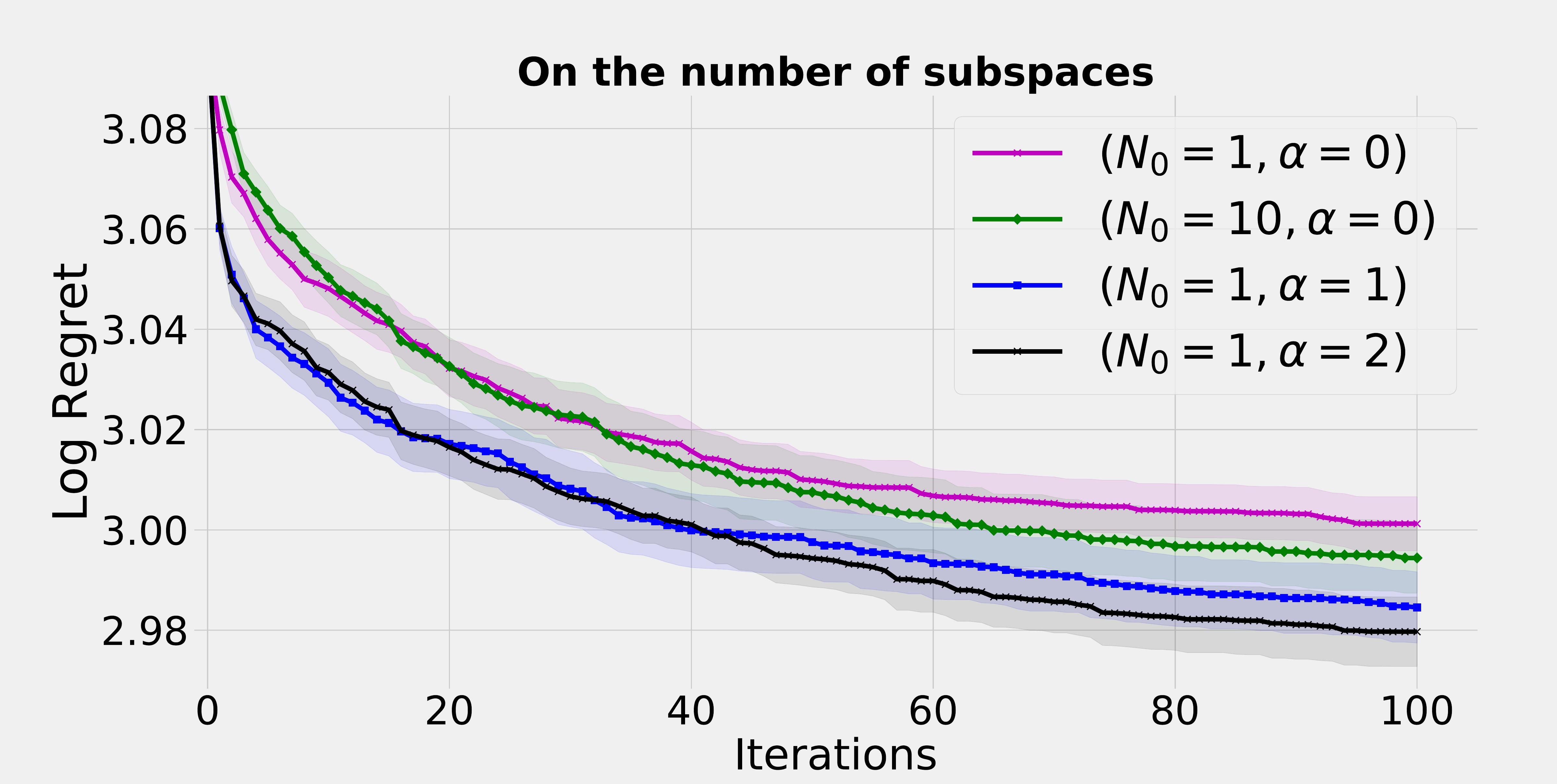}
}
\caption{Log Regret vs iterations for varying number of subspaces ($N_0$ and $\alpha$).}
\label{freq}
\end{figure}
Figure 2 shows a comparison of these variants for Ackley function and Hyper-Ellipsoid function. It indicates if we increase $N_0$ or $\alpha$ (extending the search space), the regret decreases. For the case where $\alpha = 2$, we need to generate more subspaces at each iteration (at a quadratic rate: $t^2$). This demands more computational budget to maximise the acquisition function. It may be the reason why compared to case $(N_0 =1,\alpha =0)$, the regret of case $(N_0 =1, \alpha =2)$ improved only slightly.
\vspace{-1em}
\paragraph{On subspace dimension $d$:} Dimension $d$ represents the dimension of the subspace. It directly affects the computational requirement of acquisition function maximization, and thus the BO optimization performance. As this computational requirement grows exponentially with increasing $d$, it is often computationally hard to find the global maximum in more than 10 dimensions. Thus we only study the value of $d$ up to 10, \emph{i.e.} we study cases: $d = 1, 2, 5, 10$ for problems in $D = 100$. We set $N_0 =1, \alpha = 0$ in our method. Figure 3 shows the performance of our method for different cases of $d$ for 100 dimensional Ackley function and Hyper-Ellipsoid function. It clearly indicates there is a faster convergence rate for a larger $d$, though at a higher computational cost.
\begin{figure}[t]
\centering
\subfloat[Hyper-Ellipsoid function]{\includegraphics[scale=1.0,width=.23\textwidth]{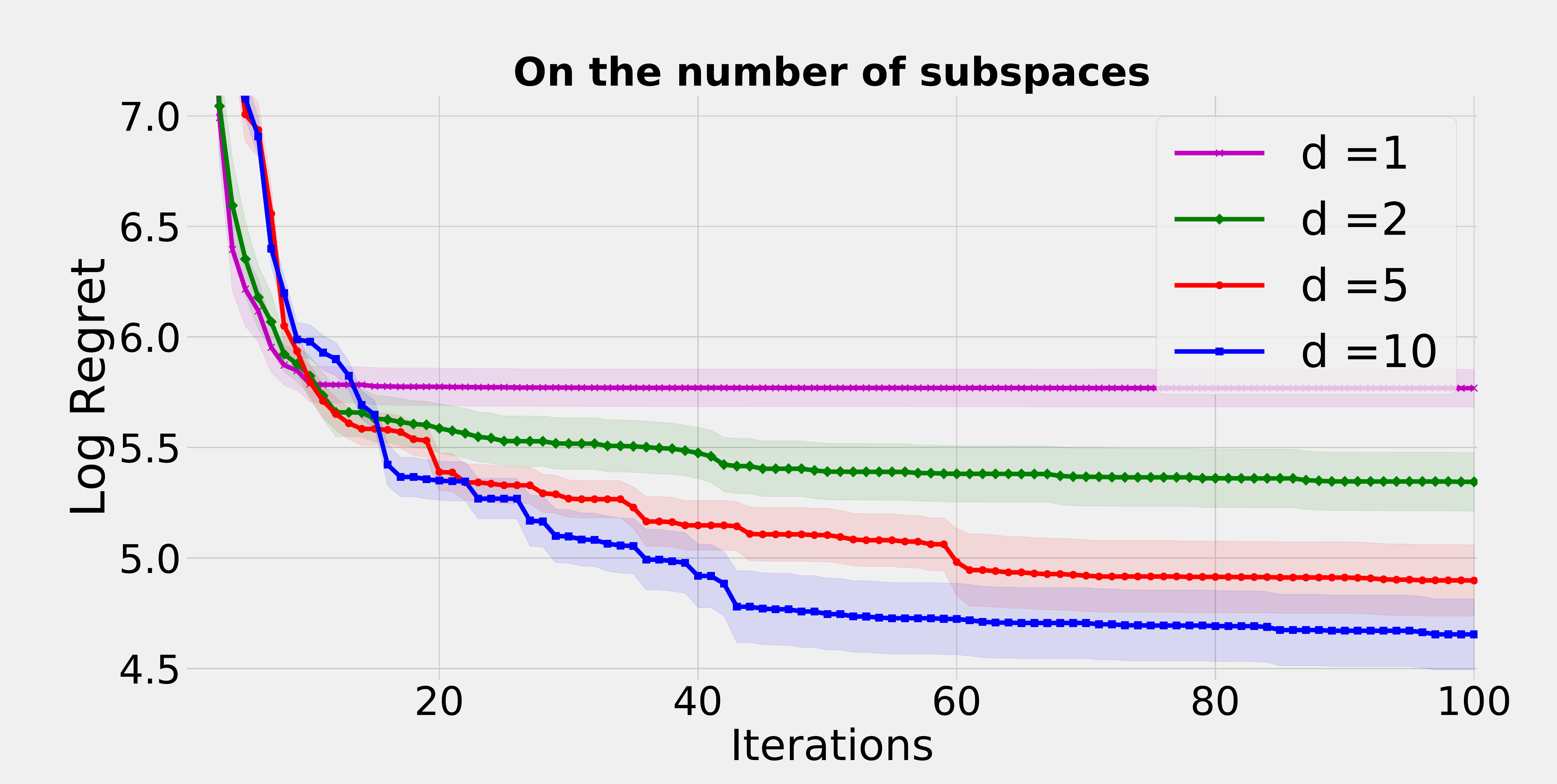}
}\hfill
\subfloat[Ackley function]{\includegraphics[scale=1.0,width=.23\textwidth]{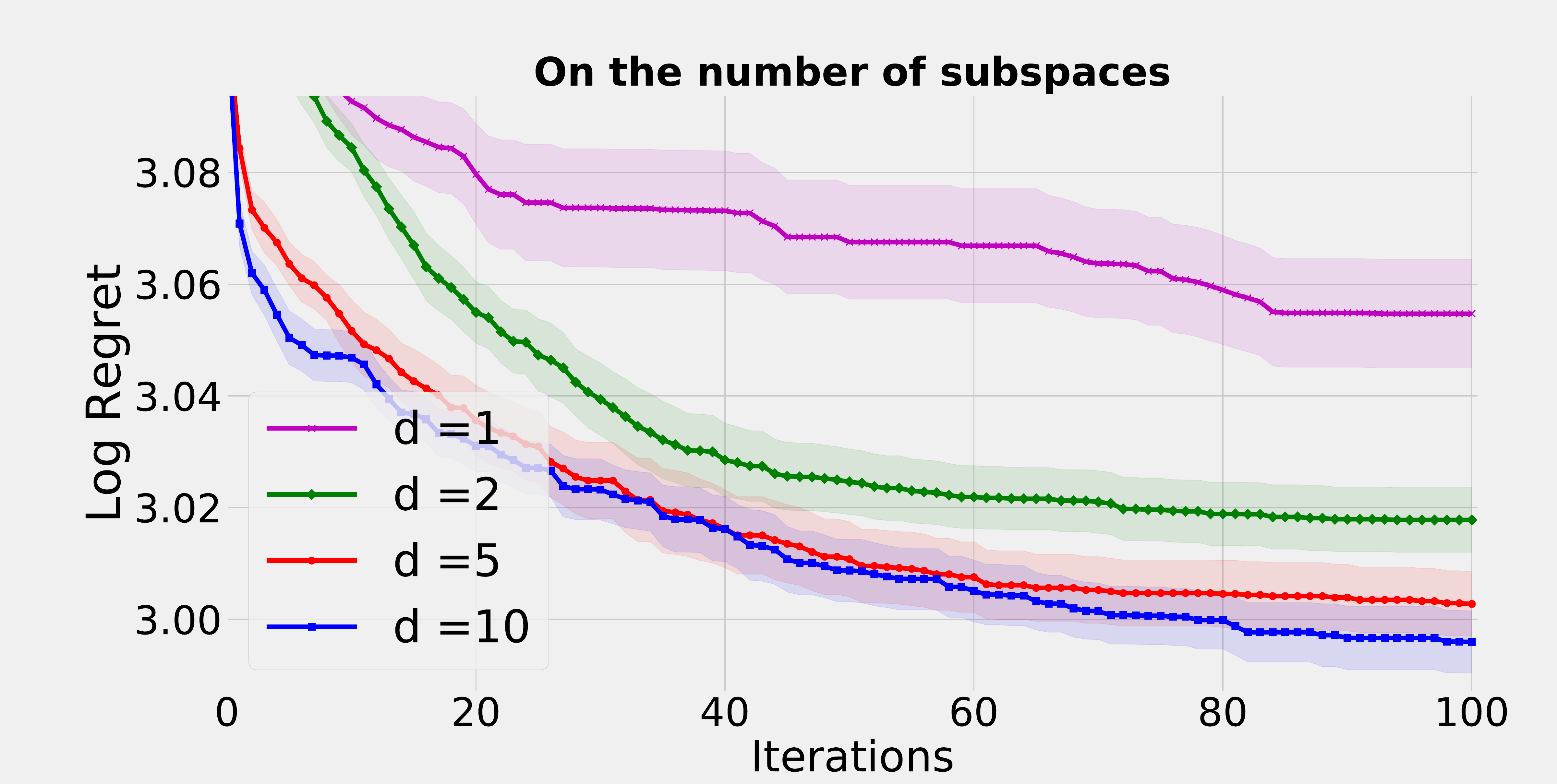}
}
\caption{Log Regret vs iterations for varying dimension $d$.}
\label{freq}
\end{figure}
\vspace{-1em}
\subsection{Learning Parameters of Machine Learning Models}
\paragraph{Neural Network Parameter Search:}
We evaluate our algorithm for learning the parameters of a neural network model as proposed by \cite{nayebi19a}. Here we are given a neural network with one hidden layer having $h$ nodes. The goal is to learn the weights between the hidden layer and the outputs in order to minimize the loss on the MNIST data set (\cite{lecun17}). We denote these weights by  $W_2$. For all experiments, $W_2$ is optimized by Bayesian optimization while the other weights and biases (we denote by $W_1$) are optimized by the Adam algorithm. We refer to \cite{OhGW18} for more details. We try two cases: $h=10$ and $h=50$. Since the network has 10 outputs (for 10 digits of MNIST), the two cases lead to optimisation in dimensions $10\times10=100$ and $50\times10=500$ respectively. Figure \ref{freq} shows the validation loss for REMBO, LineBO, SRE, HeSBO and our proposed MS-UCB. As seen from the figure, MS-UCB clearly outperforms the baselines for both cases.
\begin{figure}[t]
\centering
\subfloat{\includegraphics[scale = 1.0, width=.23\textwidth]{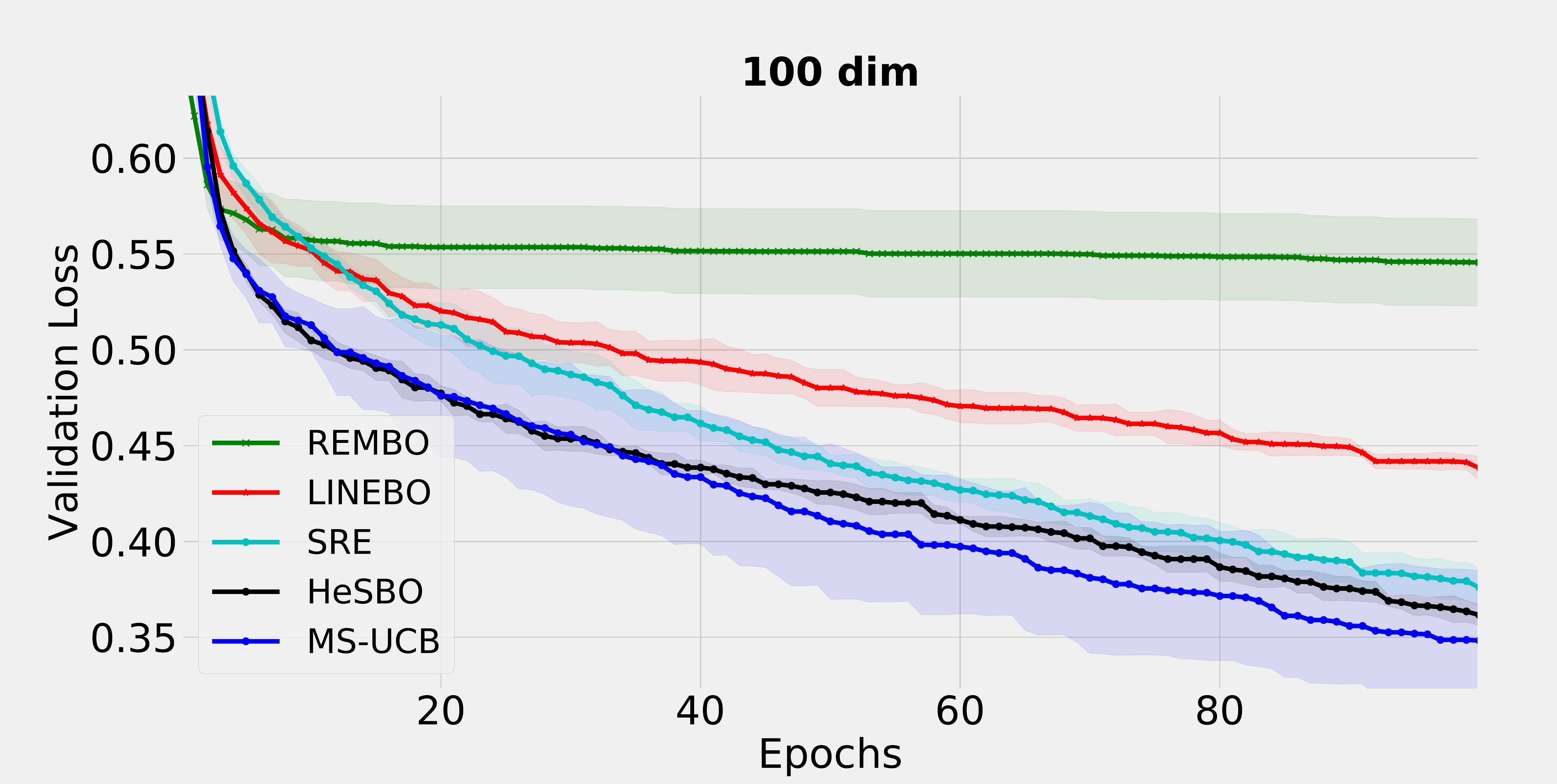}
}\hfill
\subfloat{\includegraphics[scale = 1.0, width=.23\textwidth]{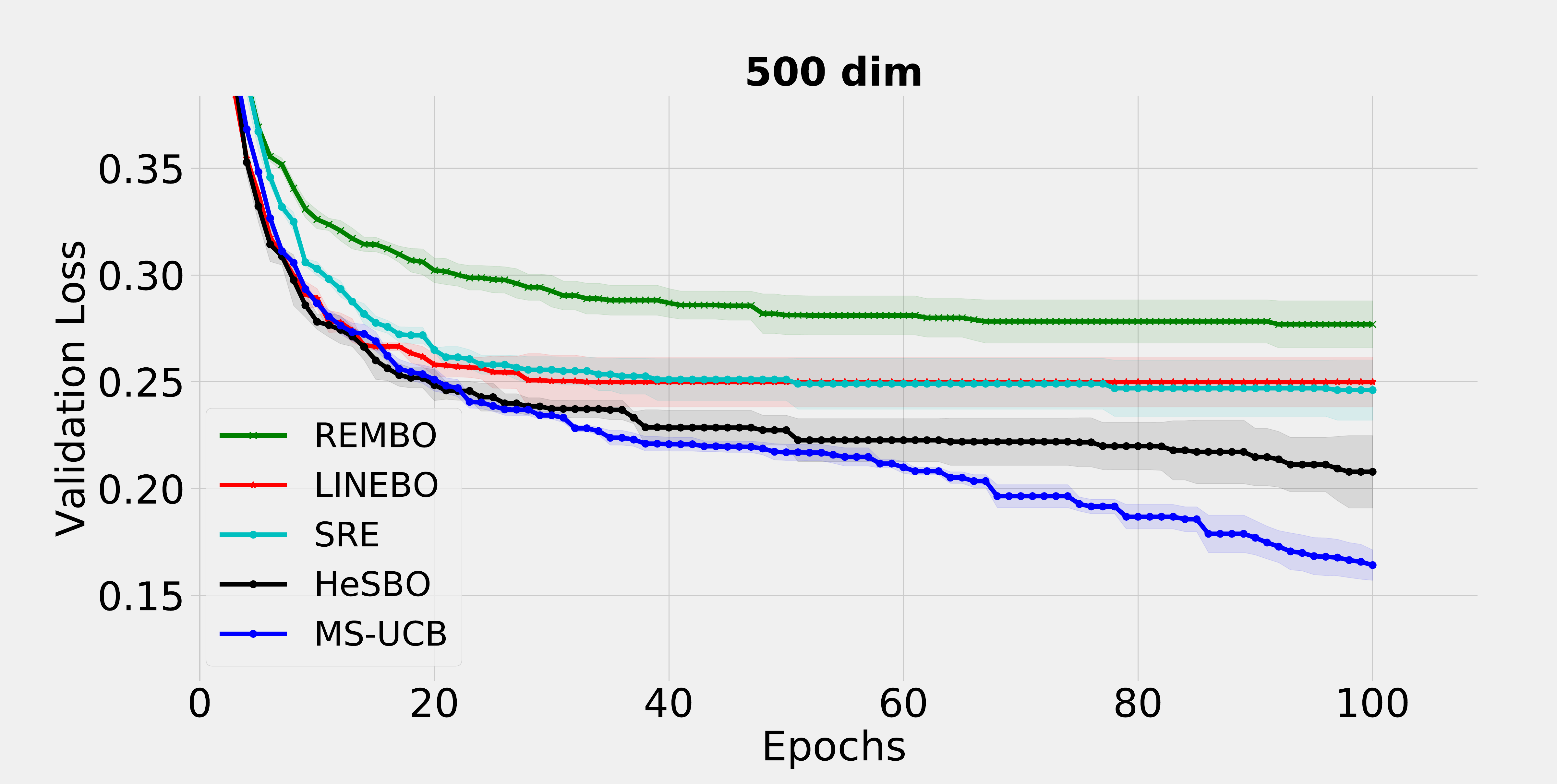}
}
\caption{The neural network benchmark with target dimension $d =10, N_0 =1, \alpha =1$ for two case where $W_2 = 10 \times 10$ and $W_2 = 50 \times 10$. For all experiments, $W_2$ is optimized by Bayesian optimization while other weights and biases are optimized by Adam algorithm.}
\label{freq}
\end{figure}
\paragraph{Learning Classification Model with Ramp Loss:}
We also test our algorithm to optimize the parameters of a classification model with a nonconvex Ramp loss, following \cite{Qian16}. The task is to find a vector $w$ and a scalar $b$ to minimize
$f(w, b) = \frac{1}{2}||w||_2^2 + C\sum_{l=1}^{L}R_s(y_l(w_Tv_l + b))$, where $v_l$ are the training instances and $y_l \in \{-1, +1\}$ are the corresponding labels. $Rs(u)=H_1(u)-H_s(u)$ with $s<1$ where $H_s(u) = \text{max}(0,s-u)$.  We use the \emph{Gisette} dataset from the UCI repository \cite{Newman98} with dimension $D= 5000$. We study the effectiveness of all algorithms fixing the hyper-parameters to $s = 0$ and $C \in \{1, 2, 5, 10\}$. We set $d = 10$ for REMBO, SRE, HeSBO and our method. For SRE, we set the number of sequential random embeddings $m=5$ as suggested in \cite{Qian16}. For our method, we consider the variant MS-UCB-1,1).
\begin{figure}[t]
\centering
\subfloat{\includegraphics[scale = 1.0, width=.23\textwidth]{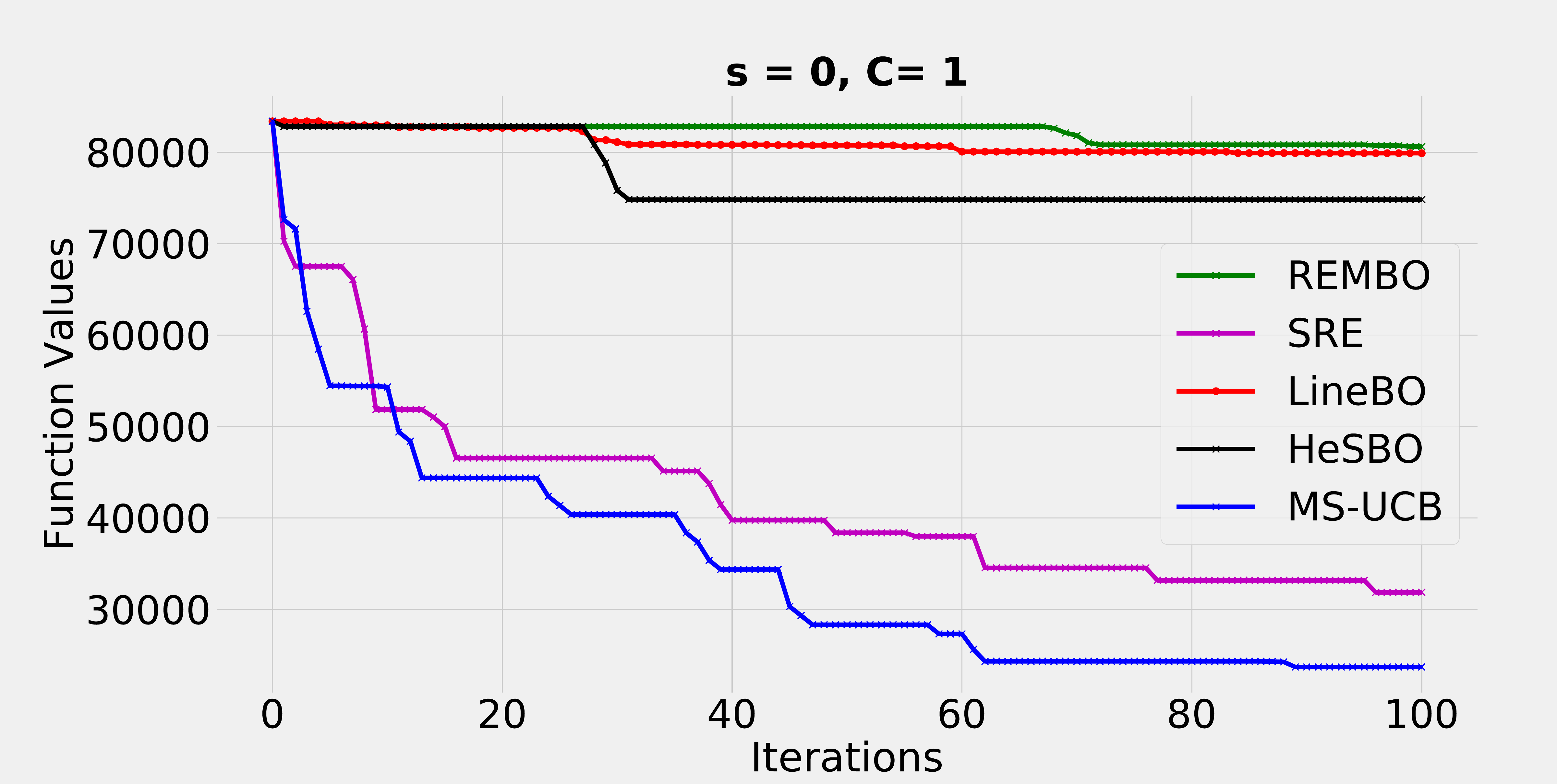}
}\hfill
\subfloat{\includegraphics[scale = 1.0, width=.23\textwidth]{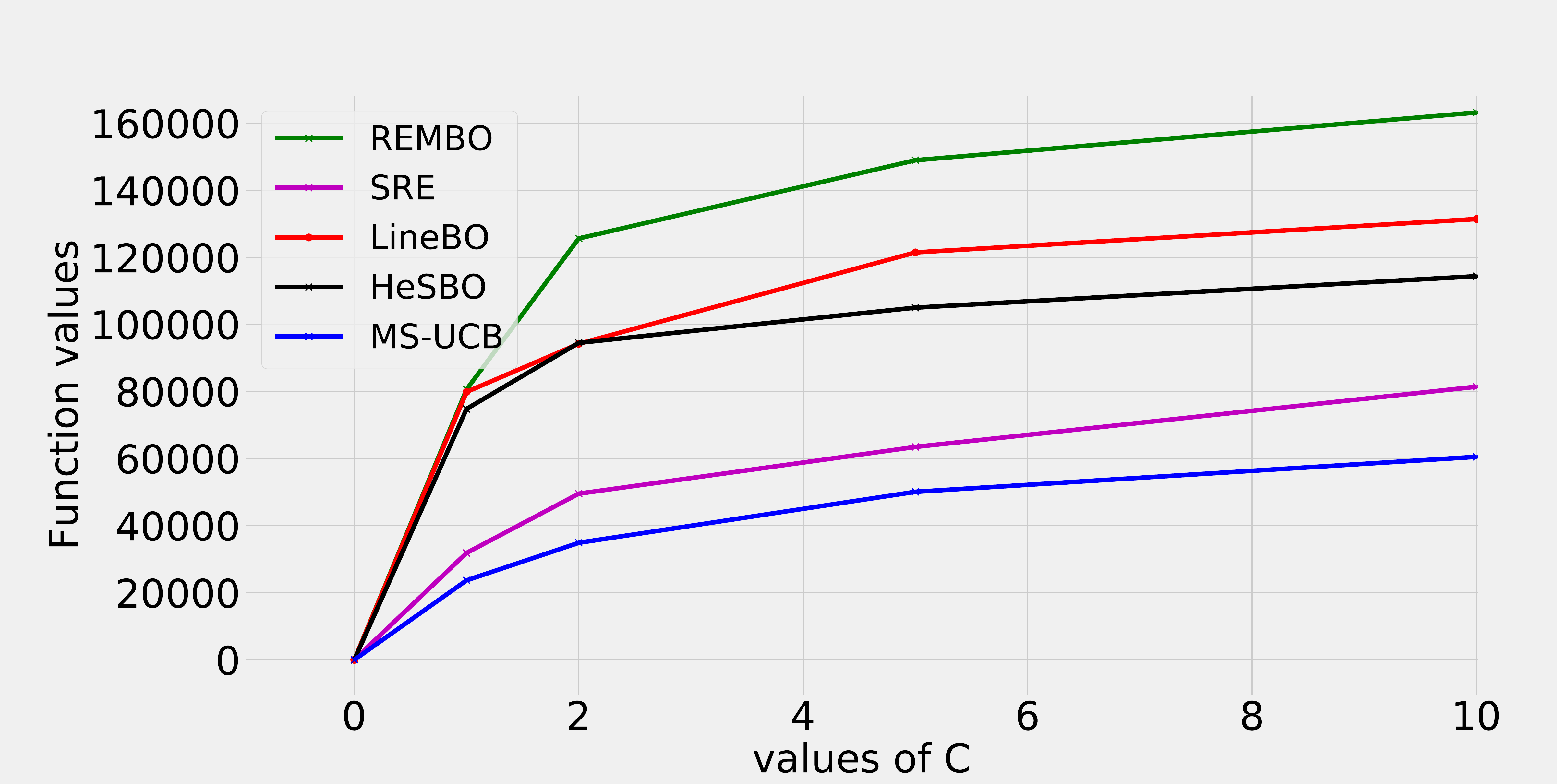}
}
\caption{Left panel shows the progress of function values vs time for $s =0$ and $C = 1$. Right panel shows the comparison of achieved loss function values against the hyperparameter $C$ of the Ramp loss.}
\label{freq}
\end{figure}
As shown in Figure 5, our method has consistently the best performance across different settings of $C$, followed by SRE. REMBO, LineBO, HESBO perform poorly for this application. This shows the effectiveness of our method.
\section{Conclusion}
We propose a  scalable Bayesian optimisation to optimise expensive blackbox functions in high dimensions. Unlike many previous existing methods, our algorithm does not make any additional assumption about the structure of the function (e.g. low effective dimension and additivity). In our method the acquisition function only requires maximisation on a discrete set of low dimensional subspaces embedded in the original high-dimensional search space and thus does not have high computational requirements for maximising acquisition functions. By varying the number of low dimensional subspaces, our algorithm has a flexibility to trade the optimisation convergence rate with the computational budget. This feature is important for many practical applications. We analyse our algorithm theoretically and show that irrespective of the number of subspaces, our algorithm always has a sublinear growth rate for cumulative regret. Further, we provide a regime for the number of subspaces where our algorithm has both tighter regret bound as well as lower computational requirement compared to the GP-UCB algorithm of \cite{Srinivas12}. We perform experiments for many optimisation problems in high dimensions and show that the sample efficiency of our algorithm is better than the existing methods given the same computational budget for optimising acquisition function.
\section{Acknowledgments}
This research was partially funded by the Australian Government through the Australian Research Council (ARC). Prof Venkatesh is the recipient of an ARC Australian Laureate Fellowship (FL170100006).
\bibliography{Bibliography-File}
\bibliographystyle{aaai}

\title{Supplementary Material}
\section{Proof for Lemma 2}
Given $\beta_t^{0}$, we have $\mathbb{P}[|f(x_t) - \mu_{t-1}(x_t)| > \sqrt{\beta_t^0}] \le e^{\beta_t^0/2}$. Since $e^{\beta_t^0/2} = \frac{6\delta}{\pi^2t^2}$. Using the union bound for $t\ \in \mathbb{N}$, we have $\mathbb{P}[|f(x_t) - \mu_{t-1}(x_t)| \le \sqrt{\beta_t^0}] \ge 1 - \frac{\pi^2\delta}{6} > 1- \delta$.
\section{Proof for Lemma 3}
We use the idea of proof of Lemma 5.7 in \cite{Srinivas12} with several modifications to adapt our method with low-dimensional subspaces. We consider the distance of any two points on subspace $\mathcal S(A, z^*_t)$:
$||(Ay + Bz^*_t) - (Ay' + Bz^*_t)||_1$, where $y, y' \in \mathcal Y$. We have matrix $A = \begin{bmatrix}
\mathbf{0}\\
\mathbf{I}
\end{bmatrix}$, thus  $||(Ay + Bz^*_t) - (Ay' + Bz^*_t)||_1 = \sum_{i =1}^{D} |[A(y-y')]_i| = \sum_{i =1}^{D-d} |[\mathbf{0}^{(D-d) \times d}(y-y')]_i)| + \sum_{i =D-d + 1}^{D} |[\mathbf{I}^{d \times d}(y-y')]_i| = \sum_{i =D-d + 1}^{D} |[(y-y')]_i|$, where we denote $[y]_i$ is the $i$th coordinate of $y$. However, since for every $1 \le i \le D-d, 1\le j \le d$ and $|[y-y']_i| \le 2$ for every $y, y' \in [-1,1]^d$, $\sum_{i =D-d + 1}^{D} |[(y-y')]_i| \le 2d$. Thus, $||(Ay + Bz^*_t) - (Ay' + Bz^*_t)||_1  \le 2d$.

On the other hand, by Assumption 1 and the union bound, for $\forall x, x' \in \mathcal S(A, z^*_t)$ and thus $ \in \mathcal X$, we have
$$|f(x) - f(x')| \le L||x -x'||_1$$
with probability greater than $1 - Dae^{-L^2/b^2}$.
Thus, by choosing $Dae^{-L^2/b^2} = \delta/2$, we have
\begin{eqnarray}
|f(x) - f(x')| \le b\sqrt{log(\frac{2Da}{\delta})}||x -x'||_1
\label{eq:100}
\end{eqnarray}
with probability greater than $1 - \delta/2$.

Now, we move to subspace $\mathcal S(A, z^*_t)$. On $\mathcal S(A, z^*_t)$,  we construct a discretization $F_t$ of size $(\tau_t)^d$ dense enough such that for any $x \in F_t$ $$||x - [x]_t]||_1 \le \frac{2d}{\tau_t}$$
where $[x]_t$ denotes the closest point in $F_t$ to $x$. We can perform this because the distance between any two point in $\mathcal S(A, z^*_t)$ is bounded by $2d$ in the above proof.

By this manner, with probability greater than $1 - \delta/2$, we have
\begin{eqnarray*}
|f(x) - f([x]_t)| & \le &  \frac{b\sqrt{log(\frac{2Da}{\delta})}||x- [x]_t||_1}{\tau_t} \\
& \le &  b\sqrt{log(\frac{2Da}{\delta})}\frac{2d}{\tau_t} \\
& < & \frac{2bd\sqrt{log(\frac{2Da}{\delta})}}{\tau_t}
\end{eqnarray*}
Let $\tau_t = 2bd\sqrt{log(\frac{2Da}{\delta})}t^2$. Thus, $|F_t| = (2bd\sqrt{log(\frac{2Da}{\delta})}t^2)^d$. We obtain
\begin{eqnarray}
|f(x) - f([x]_t)| \le \frac{1}{t^2}
\label{eq:102}
\end{eqnarray}
with probability $1- \delta/2$ for any $x \in F_t$.

Similar to Lemma 5.6 of \cite{Srinivas12}, if we set $\beta_{t}^1 = 2log(|F_t|\frac{\pi^2t^2}{3\delta}) = 2log(\frac{\pi^2t^2}{3\delta}) +  2dlog(2bd\sqrt{log(\frac{2Da}{\delta})}t^2)$, we have with probability $1 -\delta/2$, we have
\begin{eqnarray}
f(x) \le \mu_{t-1}(x) + \sqrt{\beta_{t}^1}\sigma_{t-1}(x)
\label{eq:103}
\end{eqnarray}
for any $x \in F_t$ and any $t \ge 1$. Thus, combining Eq(\ref{eq:102}) and Eq(\ref{eq:103}), if we let $x' = [Ay^* + Bz^*_t]_t$ which is the closest point in $F_t$ to $Ay^* + Bz^*_t$, we have
\begin{eqnarray*}
f^{max}_{\mathcal S_t^0} & = & f(Ay^* + Bz^*_t) \\
&\le& f([Ay^* + Bz^*_t]_t) + \frac{1}{t^2} \\
&\le& \mu_{t-1}(x') + \sqrt{\beta_{t}^{1}}\sigma_{t-1}(x') + \frac{1}{t^2}
\end{eqnarray*}
with probability $1 - \delta$. Note that since $x' \in F_t$, $x' \in \mathcal S(A, z^*_t)$.
\section{Proof for Lemma 4}
Given any $x \in \mathcal S_t^0$. Without loss of generality, we assume that $x = Ay^* + Bz$, where $z \in Z_t$. By Assumption 1 and the union bound, we have that for $\forall \mathbf{x}$,
$$|f(x^*) - f(x)| \le L||x^* -x)||_1$$
with probability greater than $1 - Dae^{-L^2/b^2}$.
Since $x^* = Ay^* + Bz^*$ and $x = Ay^* + Bz$, $||x^* -x||_1 = ||Ay^* + Bz^* - (Ay^* + Bz)||_1 = ||z - z^*||_1$. Thus,
$$|f(x^*) - f(x)| \le L||z^* -z||_1$$
with probability greater than $1 - Dae^{-L^2/b^2}$. Set $Dae^{-L^2/b^2} = \delta/2$. Thus,
\begin{eqnarray}
|f(x^*) - f(x)| \le b\sqrt{log(\frac{2Da}{\delta})}||z^* -z||_1
\label{eq:100}
\end{eqnarray}
with probability greater than $1 - \delta/2$

We denote the volume of space $\mathcal Z$ be $Vol(\mathcal Z)$. Since $z  \in [-1, 1]^{D-d}$, $Vol(\mathcal Z) = 2^{D-d}$. For $\theta > 0$, let $V_1 = \{z \in \mathcal Z = [-1, 1]^{D-d}| ||z-z^*||_1 \le \theta\}$. In the case where $V_1$ is within $\mathcal Z$, the volume $Vol(V_1)$ is the volume of a $(D-d)$-dimensional ball of radius $\theta$ with center $z^*$ and thus, $Vol(V_1) = \frac{(4\theta)^{D-d}}{\Gamma(D-d +1)}$ where $\Gamma$ is Leonhard Euler's Gamma function. The function $\Gamma(k) = (k-1)!$ for any positive integer $k$. Here, we used the formula in \cite{Xian05} for the volume of the $(D-d)$-dimensional ball in $L^1$ norms. In the worst case where $z^*$ is at the boundary of $\mathcal Z$, more precisely either $[z^*]_i = 1$ or $[z^*]_i = -1$ for any $1 \le i \le D-d$, $Vol(V_1)$ is at least $\frac{(2\theta)^{D-d}}{\Gamma(D-d +1)}$. This is because the volume halved by each dimension: $\mathcal{P}[|[z]_i - [z^*]_i| \le \theta] = \frac{\theta}{2}$ when $[z^*]_i = 1$ or $-1$. Thus, in every case, we have $Vol(V_1) \ge \frac{(2\theta)^{D-d}}{\Gamma(D-d +1)}$.

Thus, $\mathbb{P}[||z- z^*||_1 \le \theta] = \frac{Vol(V_1)}{Vol(\mathcal Z)} \ge \frac{\frac{(2\theta)^{D-d}}{\Gamma(D-d +1)}}{2^{D-d}} = \frac{\theta^{D-d}}{\Gamma(D-d +1)}$. However, $\mathbb{P}[||z- z^*||_1 > \theta] = 1 - \mathbb{P}[||z- z^*||_1 \le \theta]$. Thus, $\mathbb{P}[||z- z^*||_1 > \theta] \le 1 - \frac{\theta^{D-d}}{\Gamma(D-d +1)} \le e^{- \frac{\theta^{D-d}}{\Gamma(D-d +1)}}$, where we use the inequality $1-x \le e^{-x}$ for the last step.

Thus,
\begin{eqnarray*}
\prod_{z \in \mathcal Z_t} \mathbb{P}[||z- z^*||_1 > \theta] & \le & \prod_{z \in \mathcal Z_t} e^{- \frac{\theta^{D-d}}{\Gamma(D-d +1)}}\\
& = & e^{-|Z_t|\frac{\theta^{D-d}}{\Gamma(D-d +1)}}
\end{eqnarray*}
On the other hand, $ \mathbb{P}[||z_t^*- z^*||_1 \le \theta] = 1 - \mathbb{P}[||z_t^*- z^*||_1 > \theta]   = 1-  \mathbb{P}_{z^i_t \in \mathcal Z_t, 1 \le i \le |Z_t|}[||z_t^1- z^*||_1 > \theta \wedge...\wedge  ||z_t^{|Z_t|}- z^*||_1 > \theta]  = 1- \prod_{z^i_t \in \mathcal Z_t, 1 \le i \le |Z_t|} \mathbb{P}[||z_t^i- z^*||_1 > \theta]$. This is because $z^i_{t}$ are independent due to $z^i_t$ sampled uniformly at random in $\mathcal Z$. Thus, we have
\begin{eqnarray*}
\mathbb{P}[||z_t^*- z^*||_1  \le \theta] \ge 1 - e^{-|Z_t|\frac{\theta^{D-d}}{\Gamma(D-d +1)}}
\end{eqnarray*}
Now we set $e^{-|Z_t|\frac{\theta^{D-d}}{\Gamma(D-d +1)}}= \delta/2$ to obtain $\theta = (\Gamma(D-d +1))^{\frac{1}{D-d}}(\frac{1}{|Z_t|}log(\frac{2}{\delta}))^{\frac{1}{D-d}}$. Thus,
\begin{eqnarray}
||z_t^*- z^*||_1  \le (\Gamma(D-d +1))^{\frac{1}{D-d}}(\frac{1}{|Z_t|}log(\frac{2}{\delta}))^{\frac{1}{D-d}}
\label{eq:101}
\end{eqnarray}
with probability $1 - \delta/2$. Combining Eq(\ref{eq:100}) and Eq(\ref{eq:101}), we have
\begin{eqnarray}
|f(x^*) - f(x)| \le v_0(\frac{1}{|Z_t|}log(\frac{2}{\delta}))^{\frac{1}{D-d}}
\end{eqnarray}
with probability $1 - \delta$, where $v_0 = b\sqrt{log(\frac{2Da}{\delta})}(\Gamma(D-d +1))^{\frac{1}{D-d}}$. Thus, the lemma holds.

\section{Proof for Theorem 1}
Now, we combine the results from Lemmas 2, 3 and 4 to obtain a bound on $r_t$ as stated in the following Lemma.
\begin{proposition}[Bounding $r_t$]
Pick a $\delta \in (0,1)$ and set $\beta_t = 2log(\frac{\pi^2t^2}{\delta}) +  2dlog(2bd\sqrt{log(\frac{6Da}{\delta})}t^2)$ and set $v_0 = 2b\sqrt{log(\frac{2Da}{\delta})}(\Gamma(D-d +1))^{\frac{1}{D-d}}$, where $\Gamma(D-d +1) = (D-d+1)!$ . Then we have
\begin{eqnarray}
r_t \le 2\beta_t^{1/2}\sigma_{t-1}(x_t) + \frac{1}{t^2} + v_0(\frac{1}{|Z_t|}log(\frac{6}{\delta}))^{\frac{1}{D-d}}
\end{eqnarray}
holds with probability $\ge 1 - \delta$.
\label{lem:5.8}
\end{proposition}
\begin{proof}
We use $\frac{\delta}{3}$ for Lemmas 2,3 and 4 so that these events hold simultaneously with probability greater than $1- \delta$. Formally, by Lemma 2 using $\frac{\delta}{3}$:
\begin{eqnarray*}
f(x_t) \ge \mu_{t-1}(x_t) - \sqrt{2log(\frac{\pi^2t^2}{\delta})}\sigma_{t-1}(x_t)
\end{eqnarray*}
holds with probability $\ge 1 - \frac{\delta}{3}$. As a result,
\begin{eqnarray}
f(x_t) & \ge &  \mu_{t-1}(x_t) - \sqrt{\beta_t^0}\sigma_{t-1}(x_t) \\
& > & \mu_{t-1}(x_t) - \sqrt{\beta_t}\sigma_{t-1}(x_t) \label{eq:50}
\end{eqnarray}
holds with probability $\ge 1 - \frac{\delta}{3}$.

By Lemma 3 using $\frac{\delta}{3}$, there exists a $x' \in \mathcal S(A, z^*_t)$ such that
\begin{eqnarray*}
f^{max}_{\mathcal S_t^0} \le \mu_{t-1}(x') + \sqrt{\beta_t^{1}}\sigma_{t-1}(x') + \frac{1}{t^2}
\end{eqnarray*}
holds with probability $\ge 1 - \frac{\delta}{3}$. As a result,
\begin{eqnarray*}
f^{max}_{\mathcal S_t^0} & \le & \mu_{t-1}(x') + \sqrt{\beta_t^{1}}\sigma_{t-1}(x') + \frac{1}{t^2} \\
f^{max}_{\mathcal S_t^0} & \le & \mu_{t-1}(x') + \sqrt{\beta_t}\sigma_{t-1}(x') + \frac{1}{t^2} \\
& = & a_t(x') + \frac{1}{t^2}
\end{eqnarray*}
Since $x_t = \text{argmax}_{x \in \mathcal X_t}a_t(x)$ (defined in Algorithm 1) and $x' \in \mathcal S(A, z^*_t) \subseteq \mathcal X_t$, we have $a_t(x') \le a_t(x_t)$. Thus,
\begin{eqnarray}
 f^{max}_{\mathcal S_t^0} \le a_t(x_t) + \frac{1}{t^2} \label{eq:51}
\end{eqnarray}
holds with probability $\ge 1 - \frac{\delta}{3}$.

By Lemma 4 using $\frac{\delta}{3}$:
\begin{eqnarray}
f(x^*)- f_{\mathcal S_t^0}^{max} \le n(\frac{1}{|Z_t|}log(\frac{6}{\delta}))^{\frac{1}{D-d}} \label{eq:52}
\end{eqnarray}
holds with probability $\ge 1 - \frac{\delta}{3}$.

Combing Eq(\ref{eq:50}), Eq(\ref{eq:51}) and Eq(\ref{eq:52}), we have
\begin{eqnarray*}
r_t & = & f(x^*) - f(x_t)\\
& = & \underbrace{f(x^*) - f^{max}_{\mathcal S_t^0}}_\text{Term 1} +  \underbrace{f^{max}_{\mathcal S_t^0}}_\text{Term 2} - \underbrace{f(x_t)}_\text{Term 3}\\
& \le &  n(\frac{1}{|Z_t|}log(\frac{6}{\delta}))^{\frac{1}{D-d}}  +  \underbrace{f^{max}_{\mathcal S_t^0}}_\text{Term 2} - \underbrace{f(x_t)}_\text{Term 3}\\
& \le & n(\frac{1}{|Z_t|}log(\frac{6}{\delta}))^{\frac{1}{D-d}} + \frac{1}{t^2} + a_{t}(x_t)  - f(x_t) \label{eq:2}\\
& \le & n(\frac{1}{|Z_t|}log(\frac{6}{\delta}))^{\frac{1}{D-d}} + \frac{1}{t^2} + 2(\beta_t)^{1/2}\sigma_{t-1}(x_t)
\label{eq:3}
\end{eqnarray*}
holds with probability $\ge 1 - \delta$.
\end{proof}

The next two lemmas are used to bound $R_T$.
\begin{proposition}(Bounding $p$-series when $p >1$, \cite{tom99})
Given a $p$-series $s_n = \sum_{k =1}^{n}\frac{1}{k^p}$, where $n \in \mathbb{N}$. Then,
$$s_n < \zeta(p) < \frac{1}{p-1} + 1$$
for any $n$, where $\zeta(p) = \sum_{k =1}^{\infty}\frac{1}{k^p} $ is Euler–Riemann zeta function that always converges. For example, $\zeta(3/2) \approx 2.61$, $\zeta(2) = \frac{\pi^2}{6}$.
\label{lem:6}
\end{proposition}

\begin{proposition}(Bounding $p$-series when $p \le 1$, \cite{Chlebus2009})
Given a $p$-series $s_n = \sum_{k =1}^{n}\frac{1}{k^p}$, where $n \in \mathbb{N}$. Then,
\begin{itemize}
  \item if $p < 0$ then $1 + \frac{n^{1-p}-1}{1-p} < s_n < \frac{(n+1)^{1-p}-1}{1-p}$
  \item if $0 \le p < 1$ then $s_n < 1 + \frac{n^{1-p}-1}{1-p}$
  \item if $p = 1$ then $s_n < 1 + ln(n)$
\end{itemize}
\label{lem:5}
\end{proposition}
\begin{proposition}[Bounding $R_T$]
Pick a $\delta \in (0,1)$. Set $\beta_t = 2log(\frac{\pi^2t^2}{\delta}) +  2dlog(2bd\sqrt{log(\frac{6Da}{\delta})}t^2)$ and $v_0 = b\sqrt{log(\frac{2Da}{\delta})}(\Gamma(D-d +1))^{\frac{1}{D-d}}$, where $\Gamma(D-d +1) = (D-d+1)!$. Then, the cumulative regret of the proposed MS-UCB algorithm is bounded by
\begin{itemize}
  \item if $\alpha \ge 2(D - d)$, $R_T \le \sqrt{\beta_TC_1T\gamma_T} + 2v_0(\frac{1}{N_0}log(\frac{6}{\delta}))^{\frac{1}{D-d}} + \frac{\pi^2}{6}$
  \item if $D-d-1 \le \alpha < 2(D-d)$, $R_T \le \sqrt{\beta_TC_1T\gamma_T} + 2v_0(\frac{1}{N_0} log(\frac{6}{\delta}))^{\frac{1}{D-d}}(1 + ln(T)) + \frac{\pi^2}{6}$
  \item if $0 \le \alpha < D-d -1$, $R_T \le \sqrt{\beta_TC_1T\gamma_T} + 2v_0\frac{D-d}{D-d-\alpha -1}(\frac{1}{N_0}log(\frac{6}{\delta}))^{\frac{1}{D-d}}T^{1 -\frac{\alpha +1}{D-d}} + \frac{\pi^2}{6}$
\end{itemize}
with probability greater than $1 -\delta$, where $C_1 = 8/log(1 + \sigma^2)$, $\gamma_T$ is the
maximum information gain about the function $f$ from any set of observations of size $T$, and $L$ is the Lipschitz constant from Assumption 1 in the main paper.
\end{proposition}
\begin{proof}
By Proposition \ref{lem:5.8}, we have $R_T = \sum_{t=1}^{T} r_t = \sum_{t=1}^{T}2\beta_t^{1/2}\sigma_{t-1}(x_t) + \sum_{t=1}^{T} \frac{1}{t^2} + \sum_{1}^{T} n(\frac{1}{|Z_t|}log(\frac{1}{\delta}))^{\frac{1}{D-d}}$.

Similar to the proof of Theorem 2 in \cite{Srinivas12}, we have $\sum_{1}^{T}2\beta_t^{1/2}\sigma_{t-1}(x_t) \le \sqrt{C_1T\beta_T\gamma_T}$. $\sum_{1}^{T} \frac{1}{t^2} < \frac{\pi^2}{6}$. The remaining problem is that we need to bound $\sum_{1}^{T} v_0(\frac{1}{|Z_t|}log(\frac{1}{\delta}))^{\frac{1}{D-d}}$.

Since $N_t = N_0t^{\alpha}$, $|Z(t)| = \sum_{k=1}^{t} N_0k^{\alpha}$. Let $C = \sum_{1}^{T} n(\frac{1}{|Z_t|}log(\frac{1}{\delta}))^{\frac{1}{D-d}}$.

Since $|Z(t)| = \sum_{k=1}^{t} N_0k^{\alpha}$ and $\sum_{k=1}^{t} N_0k^{\alpha} > N_0 t^{\alpha}$, we have
\begin{eqnarray}
|Z(t)| \ge N_0 t^{\alpha} \label{eq:20}
\end{eqnarray}
On the other hand,  By Proposition \ref{lem:5} for the case where $p < 0$, we have $\sum_{k=1}^{t} k^{\alpha} > \frac{t^{\alpha +1} + \alpha}{1 + \alpha}$ as $\alpha \ge 0$. Then,
\begin{eqnarray}
|Z(t)| \ge N_0\frac{t^{\alpha +1} + \alpha}{1 + \alpha} \label{eq:21}
\end{eqnarray}
Now, we consider three cases:
\begin{itemize}
  \item $\alpha \ge 2(D-d)$. Using Eq(\ref{eq:20}), $|Z(t)| \ge N_0 t^{\alpha}$. Thus, $C < v_0(\frac{1}{N_0}
log(\frac{1}{\delta}))^{\frac{1}{D-d}} \sum_{t=1}^{T} \frac{1}{t^{\frac{\alpha}{D-d}}}$. By Proposition \ref{lem:6} for the case where $p > 1$, we have  $\sum_{1}^{T} \frac{1}{t^{\frac{\alpha}{D-d}}} < \frac{D-d}{\alpha -D +d} +1 \le 2$. Thus, $$C < 2v_0(\frac{1}{N_0}log(\frac{1}{\delta}))^{\frac{1}{D-d}}$$.

  \item $(D-d -1) \le \alpha < 2(D-d)$. Using Eq(\ref{eq:21}), $|Z(t)| \ge N_0\frac{t^{\alpha +1} + \alpha}{1 + \alpha}$.
  \begin{eqnarray*}
      C & = & \sum_{1}^{T} v_0(\frac{1}{|Z_t|}log(\frac{1}{\delta}))^{\frac{1}{D-d}} \\
        & = &  v_0(log(\frac{1}{\delta}))^{\frac{1}{D-d}} \sum_{1}^{T}(\frac{1}{|Z_t|})^{\frac{1}{D-d}} \\
        & < &  v_0(log(\frac{1}{\delta}))^{\frac{1}{D-d}} \sum_{1}^{T}(\frac{\alpha +1}{N_0 (t^{\alpha +1} + \alpha)})^{\frac{1}{D-d}}\\
        & < &  v_0(\frac{1}{N_0} log(\frac{1}{\delta}))^{\frac{1}{D-d}} (\alpha +1)^{\frac{1}{D-d}}\sum_{1}^{T}\frac{1}{t^{\frac{\alpha+1}{D-d}}}
      \end{eqnarray*}
      However, since $\alpha < 2(D-d)$, $(\alpha +1)^{\frac{1}{D-d}} < (2(D-d))^{\frac{1}{D-d}}$. $(2(D-d))^{\frac{1}{D-d}} \le 2$ by using the inequality $x \le 2^{x-1}$ with $x = D-d-1 \ge 0$. Thus, $(\alpha +1)^{\frac{1}{D-d}} < 2$. On the other hand, since $\alpha > D-d-1$
      $\sum_{1}^{T}\frac{1}{t^{\frac{\alpha+1}{D-d}}} < \sum_{1}^{T}\frac{1}{t}$. By Proposition \ref{lem:5} for the cases where $p = 1$, we have $\sum_{1}^{T} \frac{1}{t} < 1 + ln(T)$. Then,
      $$C < 2v_0(\frac{1}{N_0} log(\frac{1}{\delta}))^{\frac{1}{D-d}}(1 + ln(T))$$

  \item $0 \le \alpha < D-d-1$. Using Eq(\ref{eq:21}), $|Z(t)| \ge N_0\frac{t^{\alpha +1} + \alpha}{1 + \alpha}$. Similar to the above case, we have $C< v_0(\frac{1}{N_0} log(\frac{1}{\delta}))^{\frac{1}{D-d}} (\alpha +1)^{\frac{1}{D-d}}\sum_{1}^{T}\frac{1}{t^{\frac{\alpha+1}{D-d}}}$.   By Proposition \ref{lem:5} for the cases where $0 \le  p < 1$, we have   $\sum_{1}^{T} \frac{1}{t^{\frac{1 + \alpha}{D-d}}} < \frac{T^{1 -\frac{\alpha +1}{D-d}} -\frac{\alpha +1}{D-d}}{1-\frac{\alpha +1}{D-d}} < \frac{T^{1 -\frac{\alpha +1}{D-d}}}{1-\frac{\alpha +1}{D-d}}$. On the other hand, Since
  $\alpha < D-d-1 < 2(D-d)$, $(\alpha +1)^{\frac{1}{D-d}} < (2(D-d))^{\frac{1}{D-d}} \le 2$ as above case. Thus, $$C < 2v_0(\frac{1}{N_0}log(\frac{1}{\delta}))^{\frac{1}{D-d}}\frac{T^{1 -\frac{\alpha +1}{D-d}}}{1-\frac{\alpha +1}{D-d}}$$.
\end{itemize}
Thus, the proposition holds.
\end{proof}
Finally, we can obtain the main theorem in
\begin{theorem}[Bounding $R_T$ in $\mathcal{O}^*$]
Pick a $\delta \in (0,1)$. Then, the cumulative regret of the proposed MS-UCB algorithm is bounded by
\begin{itemize}
  \item if $\alpha \ge D-d-1$, $R_T \le \mathcal{O}^*(\sqrt{dT\gamma_T} + D)$.
  \item if $0 \le \alpha < D-d -1$, $R_T \le \mathcal{O}^*(\sqrt{dT\gamma_T} + \frac{D^2-Dd}{D-d-\alpha -1}T^{1 -\frac{\alpha +1}{D-d}})$.
\end{itemize}
with probability greater than $1 -\delta$.
\label{theorem:1}
\end{theorem}
\begin{proof}
Since $\beta _T = 2log(\frac{\pi^2T^2}{\delta}) + 2dlog(4bd\sqrt{log(\frac{12Da}{\delta})}T^2)$, $\mathcal{O}^*(\beta_T) = \mathcal{O}^*(d)$. Since $N_0 \ge 1$ and $\delta$ is small, $(\frac{1}{N_0}log(\frac{6}{\delta}))^{\frac{1}{D-d}} \le log(\frac{6}{\delta})$. For $n$, $\Gamma(D-d +1) < (D-d+1)^{D-d}$. Thus, $(\Gamma(D-d +1))^{\frac{1}{D-d}} < D-d +1 \le D$. Thus by Theorem \ref{theorem:1}, for both where $\alpha \ge 2(D-d)$ and $D-d-1 \le \alpha < 2(D-d)$, $R_T \le \mathcal{O}^*(\sqrt{dT\gamma_T} + D)$. For the case where $0 \le \alpha < D-d -1$, $R_T \le \mathcal{O}^*(\sqrt{dT\gamma_T} + \frac{D^2-Dd}{D-d-\alpha -1}T^{1 -\frac{\alpha +1}{D-d}})$.
\end{proof}
\section{Proof for Theorem 2}
\begin{theorem}[For functions having an effective subspace]
Pick a $\delta \in (0,1)$. If there exists a linear subspace $\mathcal T \subset \mathbb{R}^D$ with $d_e$ dimensions such that for all $x\in \mathbb{R}^D$, $f(x)= f(x_{\top}) $ where $x_{\top} \in \mathcal T$ is the orthogonal projection of $x$ onto $\mathcal T$, then
\begin{itemize}
  \item if $\alpha \ge d_e - 1$, $R_T \le \mathcal{O}^*(\sqrt{dT\gamma_T} + d_e)$.
  \item if $0 \le \alpha < d_e -1$, $R_T \le \mathcal{O}^*(\sqrt{dT\gamma_T} + \frac{d_e^2}{d_e-\alpha -1}T^{1 -\frac{\alpha +1}{d_e}})$.
\end{itemize}
with probability greater than $1 -\delta$.
\end{theorem}
\begin{proof}
To bound $r_t$, we still use same results from Lemmas 2 and 3 and a modification of Lemma 4 as follows.

Given any $x \in \mathcal S_t^0$, by Assumption 1 and the union bound, we have that for $\forall \mathbf{x}$,
$$|f(x^*) - f(x)| \le L||x^* -x||_1$$
with probability greater than $1 - Dae^{-L^2/b^2}$. Set $Dae^{-L^2/b^2} = \delta/2$. Thus,
\begin{eqnarray}
|f(x^*) - f(x)| \le b\sqrt{log(\frac{2Da}{\delta})}||x^* -x||_1
\label{eq:105}
\end{eqnarray}
with probability greater than $1 - \delta/2$.

We denote the volume of space $\mathcal X$ be $Vol(\mathcal X)$. By the argument similar to proof of Proportion 1 in \cite{Johannes19} that we restrict the variation of the function only in $d_e$ active dimensions, thus we only care about $Vol(\mathcal X) = 2^{d_e}$ as $[x]_i \in [-1,1]$ for every $1 \le i \le D$. Note that we consider the volume of space $\mathcal X$ instead of space $\mathcal Z$ as in the proof of Lemma 4. For any $\theta > 0$, let $V_1 = \{x \in \mathcal X = [-1, 1]^{D}| ||x-x^*||_1 \le \theta\}$. We approximate the volume of $V_1$ to the volume of the active subspace. When $V_1$ is within $\mathcal X$, the volume $Vol(V_1)$ is the volume of a $d_e$-dimensional ball of radius $\theta$ with center $x^*$ and thus, $Vol(V_1) = \frac{(4\theta)^{d_e}}{\Gamma(d_e + 1)}$ where $\Gamma$ is Leonhard Euler's Gamma function. The function $\Gamma(k) = (k-1)!$. In the worst case where $x^*$ is at the boundary of $\mathcal X$, $Vol(V_1)$ is at least $\frac{(2\theta)^{d_e}}{\Gamma(d_e +1)}$. This is because the volume halved by each active dimension. Thus, in every case, we have $Vol(V_1) \ge \frac{(2\theta)^{d_e}}{\Gamma(d_e +1)}$.

Thus, $\mathbb{P}[||x- x^*||_1 \le \theta] = \frac{Vol(V_1)}{Vol(\mathcal X)} \ge \frac{\frac{(2\theta)^{d_e}}{\Gamma(d_e +1)}}{2^{d_e}} = \frac{\theta^{d_e}}{\Gamma(d_e +1)}$. However, $\mathbb{P}[||x- x^*||_1 > \theta] = 1 - \mathbb{P}[||x- x^*||_1 \le \theta]$. Thus, $\mathbb{P}[||x- x^*||_1 > \theta] \le 1 - \frac{\theta^{d_e}}{\Gamma(d_e +1)} \le e^{- \frac{\theta^{d_e}}{\Gamma(d_e +1)}}$, where we use the inequality $1-x \le e^{-x}$ for the last step.

Thus,
\begin{eqnarray*}
\prod_{x \in \mathcal S_t^0} \mathbb{P}[||x- x^*||_1 > \theta] & \le & \prod_{x \in \mathcal S_t^0} e^{- \frac{\theta^{d_e}}{\Gamma(d_e +1)}}\\
& = & e^{-|Z_t|\frac{\theta^{d_e}}{\Gamma(d_e +1)}}
\end{eqnarray*}
On the other hand, let $x_t^* = Ay^* + Bz_t^*$. We have $ \mathbb{P}[||x_t^*- x^*||_1 \le \theta] = 1 - \mathbb{P}[||x_t^*- x^*||_1 > \theta]   = 1-  \mathbb{P}_{x^i_t \in \mathcal S_t^0, 1 \le i \le |Z_t|}[||x_t^1- x^*||_1 > \theta \wedge...\wedge  ||x_t^{|Z_t|}- x^*||_1 > \theta]  = 1- \prod_{x^i_t \in \mathcal S_t^0, 1 \le i \le |Z_t|} \mathbb{P}[||x_t^i- x^*||_1 > \theta]$. This is because $x^i_{t}$ are independent due to $z^i_t$ sampled uniformly at random in $\mathcal Z$. Thus, we have
\begin{eqnarray*}
\mathbb{P}[||x_t^*- x^*||_1  \le \theta] \ge 1 - e^{-|Z_t|\frac{\theta^{d_e}}{\Gamma(d_e+1)}}
\end{eqnarray*}
Now we set $e^{-|Z_t|\frac{\theta^{d_e}}{\Gamma(d_e +1)}}= \delta/2$ to obtain $\theta = (\Gamma(D-d +1))^{\frac{1}{d_e}}(\frac{1}{|Z_t|}log(\frac{2}{\delta}))^{\frac{1}{d_e}}$. Thus,
\begin{eqnarray}
||x_t^*- x^*||_1  \le (\Gamma(d_e +1))^{\frac{1}{d_e}}(\frac{1}{|Z_t|}log(\frac{2}{\delta}))^{\frac{1}{d_e}}
\label{eq:106}
\end{eqnarray}
with probability $1 - \delta/2$. Combining Eq(\ref{eq:105}) and Eq(\ref{eq:106}), we have
\begin{eqnarray}
|f(x^*) - f(x)| \le v_d(\frac{1}{|Z_t|}log(\frac{2}{\delta}))^{\frac{1}{d_e}}
\end{eqnarray}
with probability $1 - \delta$, where $v_d = b\sqrt{log(\frac{2Da}{\delta})}(\Gamma(d_e +1))^{\frac{1}{d_e}}$, where $\Gamma(d_e +1) = (d_e+1)!$.

The remaining proof is similar to Theorem \ref{theorem:1}. Thus, the result of Theorem \ref{theorem:1} is applied for this case of the effective dimension by replacing $D-d$ dimensions down $d_e$ dimensions.
\end{proof}
\section{Proof for Theorem 3}
\begin{theorem}[For $d =0$]
Pick a $\delta \in (0,1)$. Then, the cumulative regret of the proposed MS-UCB algorithm is bounded by
\begin{itemize}
  \item if $\alpha \ge D-1$, $R_T \le \mathcal{O}^*(\sqrt{T\gamma_T} + D)$.
  \item if $0 \le \alpha < D-1$, $R_T \le \mathcal{O}^*(\sqrt{dT\gamma_T} + \frac{D^2}{D-\alpha -1}T^{1 -\frac{\alpha +1}{D}})$.
\end{itemize}
with probability greater than $1 -\delta$.
\label{theorem:3}
\end{theorem}
\begin{proof}
  The proof is very similar to our case where $1 \le d < D$ except that to bound $f_{\mathcal S_t^0}^{max}$, we need to modify Lemma 3 in the main paper. It restates that
  Given a $\delta \in (0,1)$ and set $\beta_t^1 = 2log(\frac{\pi^2t^2}{6\delta})$. Then we have
\begin{eqnarray}
f_{\mathcal S_t^0}^{max} \le \mu_{t-1}(z_t^*) + \sqrt{\beta_t^1} \sigma_{t-1}(z_t^*)
\end{eqnarray}
holds with probability $\ge 1 - \delta$. Recall that $\mathcal S_t^0 = \{z_i\}_{z_i \in Z_t}$ in this case. This is because the subspace $\mathcal S(A, z_t^*)$ is degenerated to 0. Thus, to bound $f_{\mathcal S_t^0}^{max}$, we just need a result similar to Lemma 5.5 of \cite{Srinivas12}.
Finally, the result of Theorem \ref{theorem:1} is applied to this case $d =0$ by replacing $D-d$ by $D$ and parameter $\beta_t$ by $4log(\frac{\pi^2t^2}{2\delta})$.
\end{proof}
\end{document}